\documentclass[letterpaper]{article}
\usepackage[letterpaper]{geometry}

\usepackage[utf8]{inputenc}
\usepackage[T1]{fontenc} 
\usepackage[cjk]{kotex}
\usepackage{amsmath}
\usepackage{amssymb}
\usepackage{graphicx}
\usepackage{url}
\usepackage{xcolor}
\usepackage{latexsym}

\usepackage{hyperref}
\definecolor{darkblue}{rgb}{0, 0, 0.5}
\hypersetup{colorlinks=true,citecolor=darkblue, linkcolor=darkblue, urlcolor=darkblue}
\usepackage{tabularx}
\usepackage{datetime}
\usepackage{arydshln}

\usepackage{setspace}

\usepackage{booktabs}

\usepackage[authoryear,round]{natbib}
\usepackage{algorithmic,algorithm}
\renewcommand{\algorithmiccomment}[1]{\bgroup\hfill$\triangleright$~#1\egroup}

\usepackage{amsthm}
\newtheorem{theorem}{Theorem}

\newtheorem{lemma}{Lemma}

\newtheorem{corollary}{Corollary}

\title{A gentle tutorial on Bock's algorithm for minimum directed spanning trees with a structured reformulation\thanks{The phrase \emph{a gentle tutorial} is inspired by the report \emph{A Gentle Tutorial of the EM Algorithm and Its Application to Parameter Estimation for Gaussian Mixture and Hidden Markov Models} \citep{bilmes-1998-gentle}.}}

\author{
\begin{tabular}{ccc}
 Yuxi Wang    &&  Jungyeul Park\\
 UBC    &&  KAIST \\
Canada    &&  South Korea \\
{\tt johnyuxiwang@students.ubc.ca} 
&$\qquad$& {\tt jungyeul@kaist.ac.kr} 
\end{tabular}
}
\date{ 
}

\begin{document}

\maketitle

\begin{abstract}
Bock's 1971 algorithm is an exact primal--dual method for the minimum-cost arborescence problem, but its Algol presentation obscures the interaction of its maintained arrays and label-directed control flow. We provide a self-contained tutorial comprising the original listing, a line-mapped explanation, a circuit-forming three-node example, and a complete trace of Bock's ten-node instance. We also present a structured reformulation that replaces temporary span-label changes with explicit component and trace state. Local tightness and contraction-progress results, together with an operational-correspondence theorem, establish that the reformulation preserves Bock's candidate choices, transfers, and final solution.
\end{abstract}

\tableofcontents

\doublespacing

\section{Introduction}

Finding a minimum-cost branching, or arborescence, in a directed graph is a classical problem in combinatorial optimization. Given a directed graph with a designated origin and a cost assigned to each feasible link, the problem is to select a rooted spanning arborescence of minimum total cost. The algorithms of \citet{chu-liu-1965-on-the-shortest} and \citet{edmonds-1967-optimum}, together with later refinements \citep{tarjan-1977-finding,camerini-fratta-maffioli-1979-a-note}, provide the most familiar exact solutions.

Bock's algorithm \citep{bock-1971-an-algorithm} is an exact primal--dual procedure for the same problem. Rather than constructing a sequence of reduced graphs containing explicit supernodes, it retains the original cost matrix and node indices throughout execution. Provisional links are recorded in \(I^{*}\), contracted circuits are represented by common \(\mathrm{SPAN}\) labels, and backward traces and exchanges are maintained through \(\bar I\) and \(\bar J\). Column adjustments in \(U_1\) determine the reduced costs used to select links entering the current span. The original presentation is difficult to reconstruct because it is written in Algol, uses label-directed control flow, and leaves the functions of these arrays implicit in their interaction.

This paper provides a self-contained account of Bock's procedure. We reproduce the original Algol listing, define its matrix conventions and maintained state, and give a label-by-label explanation that preserves its variables, tests, assignment order, and control flow. A three-node example isolates circuit formation, contraction, and the ensuing star-and-bar transfer. We then trace Bock's original ten-node matrix from initialization to termination, extending the partial computation in the source paper into a complete account of every candidate search, column adjustment, backward trace, contraction, and transfer required to obtain the reported optimum.

We further give a structured reformulation embodied by a readable implementation of the procedure. It replaces temporary negation of span labels with an explicit trace mask and represents contraction by merging visited components into the active component, while retaining Bock's candidate scan, column adjustments, starred predecessors, barred references, and transfer assignments. We prove the immediate tightness of each selected candidate, establish strict progress under contraction, and show operational correspondence between the structured formulation and the Algol procedure under the same scan order and tie-breaking convention. The resulting presentation makes Bock's algorithm independently readable, executable, and verifiable while preserving the computational organization of the original method.

\paragraph{Availability}
The accompanying repository, available at
\url{https://github.com/jungyeul/mst-bock/}, contains an
Algol-faithful Python implementation of Bock's procedure and the
structured reformulation presented in Section~\ref{sec:reformulation}.
A scanned copy of Bock's original paper is available at
\url{https://tinyurl.com/594b96ca}.

\section{The problem}

\paragraph{Minimum directed spanning tree rooted at an origin}
Let \(G=(V,E)\) be a directed graph with \(V=\{1,\dots,N\}\) and designated origin \(j_0\). Following Bock, we use \emph{link} for a directed arc and represent its cost by an \(N\times N\) matrix \(C\). The entry \(C[I,J]\) is the cost of the link \(I\to J\), where \(I\) is its initial node and \(J\) its terminal node. Thus, row \(I\) contains the costs of links leaving \(I\), and column \(J\) contains the costs of links entering \(J\).

Bock uses \(C[I,J]=M\) to denote an infeasible link. The value \(M\) is symbolic infinity rather than an admissible cost: links assigned \(M\) are excluded from candidate selection, and \(M\) is treated as greater than every finite value occurring in the procedure. The feasible link set can therefore be written as
\[
E=\{(I,J)\in V\times V:I\neq J,\ J\neq j_0,\ C[I,J]<M\}.
\]

A directed spanning tree rooted at \(j_0\), or rooted arborescence, is a set \(T\subseteq E\) such that every node \(J\neq j_0\) has exactly one incoming link, \(j_0\) has no incoming link, and every node is reachable from \(j_0\) by a directed path. These conditions imply that \(T\) contains \(N-1\) links and no directed circuit. If \(\mathcal{A}(G,j_0)\) denotes the set of rooted arborescences, the optimization problem is
\[
T^{*}\in
\arg\min_{T\in\mathcal{A}(G,j_0)}
\sum_{(I,J)\in T} C[I,J].
\]
If \(\mathcal{A}(G,j_0)\) is empty, the instance is infeasible.

\paragraph{Integer-programming formulation}
For each feasible link \((I,J)\in E\), let \(x_{IJ}=1\) if \(I\to J\) is selected and \(x_{IJ}=0\) otherwise. The minimum-cost arborescence problem can be written as
\begin{align}
\min_x\quad
&\sum_{(I,J)\in E} C[I,J]x_{IJ}
\label{eq:obj}\\
\text{subject to}\quad
&\sum_{\substack{I:(I,J)\in E}}x_{IJ}=1
&&\text{for every }J\in V\setminus\{j_0\},
\label{eq:in-degree}\\
&\sum_{\substack{(I,J)\in E\\I\notin S,\ J\in S}}x_{IJ}\geq1
&&\text{for every nonempty }S\subseteq V\setminus\{j_0\},
\label{eq:rooted-cut}\\
&x_{IJ}\in\{0,1\}
&&\text{for every }(I,J)\in E.
\label{eq:binary}
\end{align}

Constraint~\eqref{eq:in-degree} selects exactly one incoming link for each nonorigin node. Constraint~\eqref{eq:rooted-cut} requires at least one selected link to enter every nonempty set of nodes not containing the origin. A directed circuit disconnected from \(j_0\) would violate this condition because every node in the circuit would already receive its unique incoming link from within the set. The two constraint families therefore ensure that the selected links form an arborescence rooted at \(j_0\).

\paragraph{Linear-programming dual}
Replacing the binary condition in \eqref{eq:binary} by \(x_{IJ}\geq0\) gives the standard linear relaxation. The arborescence polytope defined by \eqref{eq:in-degree} and \eqref{eq:rooted-cut} is integral, so an optimal extreme point of this relaxation represents an arborescence.

Associate a free variable \(\pi_J\) with the equality constraint for each \(J\neq j_0\), and a nonnegative variable \(\lambda_S\) with each rooted-cut constraint. The dual is
\begin{align}
\max_{\pi,\lambda}\quad
&\sum_{J\neq j_0}\pi_J
 +\sum_{\emptyset\neq S\subseteq V\setminus\{j_0\}}\lambda_S
\label{eq:dual-obj}\\
\text{subject to}\quad
&\pi_J+
\sum_{\substack{\emptyset\neq S\subseteq V\setminus\{j_0\}\\
I\notin S,\ J\in S}}
\lambda_S
\leq C[I,J]
&&\text{for every }(I,J)\in E,
\label{eq:dual-link}\\
&\lambda_S\geq0
&&\text{for every nonempty }S\subseteq V\setminus\{j_0\}.
\label{eq:dual-sign}
\end{align}

The constraint in \eqref{eq:dual-link} assigns a price to the terminal node \(J\) and to every priced set entered by \(I\to J\). Its slack is the difference between \(C[I,J]\) and the sum of these prices. Complementary slackness requires every selected link to be tight with respect to an optimal dual solution.

\paragraph{Relation to Bock's maintained state}
Bock does not store the variables \(\pi_J\) and \(\lambda_S\) separately. The array \(U_1[J]\) stores the accumulated column adjustment used in the algorithm's reduced-cost comparisons, while common \(\mathrm{SPAN}\) labels identify the contracted node sets to which successive adjustments apply. For the current active span, L4 examines links \(I\to J\) satisfying
\[
J\leq K,\qquad
\mathrm{SPAN}[J]=H,\qquad
\mathrm{SPAN}[I]\neq H,
\]
and compares their reduced costs
\[
C[I,J]-U_1[J].
\]
If \(I_1\to J_1\) is the selected candidate, then
\[
DU=C[I_1,J_1]-U_1[J_1].
\]
At L5, \(DU\) is added to \(U_1[J]\) for every column \(J\leq K\) in the active span. Immediately after this update,
\[
C[I_1,J_1]-U_1[J_1]=0.
\]

The starred array \(I^{*}\) records the current primal selection: \(I^{*}[J]=I\) means that \(I\to J\) is provisionally selected. When these links form a directed circuit, L7 assigns a common \(\mathrm{SPAN}\) label to the participating spans. The bar arrays preserve the candidate links needed to revise the starred selection when the contracted circuit is resolved. In this way, Bock carries out the primal selection, reduced-cost adjustment, contraction, and exchange operations without explicitly storing the full family of cut variables in \eqref{eq:dual-obj}--\eqref{eq:dual-sign}.

The existing verbatim Algol figure should remain unchanged. Replace the surrounding input/state and description material with the following; insert the existing Figure~`\ref{algorithm1}` at the marked location. The three-node example should remain outside this section until it is revised.

\section{The algorithm}
\label{sec:algorithm}

\subsection{Original Algol procedure}

Figure~\ref{algorithm1} reproduces Bock's original Algol procedure verbatim. Its statement labels L1--L100 determine the control flow and assignment order used throughout this paper. L1, L97, and L100 handle input and output; L2--L11 contain the optimization procedure; L98 and L99 report infeasible and optimal termination, respectively.

\begin{figure}
\singlespacing
\centering

\scriptsize
\begin{tabular}{@{}p{1.5cm}p{\dimexpr\textwidth-1cm-2\tabcolsep\relax}@{}}

\underline{BEGIN}&  \underline{INTEGER} EX, N, I, J, J0, I1 , J1, I2, J2, K, H, H1, SS, M, DU, Z;\\
& \underline{INTEGER ARRAY} U1, I STAR, I BAR, J BAR, SPAN[1:50], C[1:50, 1:50]; \\
& \underline{PROCEDURE} INPUT; \underline{PROCEDURE} OUTPUT; \\

L1: & INPUT (EX, N, J0, M);\\
    & \underline{FOR} I := 1 \underline{STEP} 1 \underline{UNTIL} N \underline{DO} 
        \underline{FOR} J := 1 \underline{STEP} 1  \underline{UNTIL} N \underline{DO} 
            INPUT (C[I, J]);\\
    & GO TO L97; \\

L2: &   \underline{FOR} J:=1 \underline{STEP} 1 \underline{UNTIL} N \underline{DO}\\
    & $\quad$ \underline{BEGIN} 
                U1[J] := I STAR[J] := I BAR[J] := J BAR[J] := 0;
                SPAN[J] := J
            \underline{END}; \\
    & SS := N; K := Z := 0;\\

L3: & \underline{IF} K = N \underline{THEN} \underline{GO TO} L99; K := K + 1;\\
    & \underline{IF} K = J0 \underline{THEN} \underline{GO TO} L3;\\

~& \\
\underline{COMMENT} & STATEMENTS L4 THROUGH L5 FIND A LEAST-COST (CANDIDATE) LINK IN THE COLUMNWISE COMPLEMENT OF THE LARGEST SPANNED SUBMATRIX INTERSECTED BY COLUMN K; \\
L4: & DU := M; H := SPAN[K];\\
    & \underline{FOR} J := 1 \underline{STEP} 1 \underline{UNTIL} K \underline{DO} 
            \underline{IF} SPAN[J] = H \underline{THEN}\\
    & $\quad$ \underline{FOR} I:= 1 \underline{STEP} 1 \underline{UNTIL} N \underline{DO}\\
    & $\quad$$\quad$ \underline{IF} SPAN[I] != H $\land$ C[I,J] < M $\land$  C[I,J] - U1[J] < DU \underline{THEN}\\
    & $\quad$$\quad$$\quad$  \underline{BEGIN} 
                            DU := C[I,J] - U1[J];
                            I1 := I; J1 := J
                        \underline{END};\\
    & \underline{IF} DU = M \underline{THEN} \underline{GO TO} L98;\\

L5: & \underline{FOR} J := 1 \underline{STEP} 1 \underline{UNTIL} K \underline{DO} \underline{IF} SPAN[J] = H \underline{THEN} U1[J] := U1[J] + DU;\\
~& \\

\underline{COMMENT} & STATEMENTS L6 THROUGH L8 TRACE BACKWARD IN THE STARRED  PARTIAL TREE LEADING TO THE CANDIDATE LINK UNTIL A CIRCUIT OR THE ROOT OF THE PARTIAL TREE IS FOUND; \\

L6: &   J :=I1; \\
L7: & \underline{IF} SPAN[J] = H \underline{THEN}
            \underline{BEGIN} 
                SS := SS + l;
                \underline{FOR} J := 1 \underline{STEP} 1 \underline{UNTIL} K \underline{DO}
                    \underline{IF} SPAN[J] = H $\lor$ SPAN[J] < 0 \underline{THEN} 
                        SPAN[J] := SS;
                \underline{GO TO} L4
            \underline{END}; \\

L8: &    \underline{IF} I STAR[J] > 0 \underline{THEN}\\ 
& $\quad$\underline{BEGIN} \\
& $\quad$$\quad$ \underline{IF} SPAN[J] > 0 \underline{THEN}\\
& $\quad$$\quad$$\quad$ \underline{BEGIN} H1 := SPAN[J];\\
& $\quad$$\quad$$\quad$$\quad$ \underline{FOR} J2 := 1 \underline{STEP} 1 \underline{UNTIL} K \underline{DO} \underline{IF} SPAN[J2] = H1 \underline{THEN} SPAN[J2] := -SPAN[J2] \underline{END};\\
& $\quad$$\quad$ \underline{IF} I BAR[J] = 0 \underline{THEN}\\
& $\quad$$\quad$$\quad$ \underline{BEGIN}  I BAR[J] := I1; J BAR[J] : = J1 \underline{END};\\
& $\quad$$\quad$ J := I STAR[J]; \underline{GO TO} L7\\
& $\quad$ \underline{END};\\

~& \\
\underline{COMMENT}& STATEMENTS L9 THROUGH L11 STAR THE CANDIDATE LINK, MAKE NECESSARY TRANSFERS OF STARS AND BARS, AND UPDATE REFERENCES OF STARS TO BARS;\\

L9:&  \underline{FOR} J := 1 \underline{STEP} 1 \underline{UNTIL} K \underline{DO}
            \underline{IF} SPAN[J] < 0 \underline{THEN} 
                SPAN[J] := -SPAN[J]; \\
    & I2 := J2 := 0; \\

L10:&  \underline{FOR} J := 1 \underline{STEP} 1 \underline{UNTIL} K \underline{DO}\\
& $\quad$ \underline{IF} I BAR[J] = I1 $\land$ J BAR[J] = J1 \underline{THEN} \underline{BEGIN} I BAR[J] := I2; J BAR[J] := J2 \underline{END};\\
&         I := I BAR[J1]; J := J BAR[J1];
        I BAR[J1] := I2; J BAR[J1] := J2;
        I2 := I STAR[J1]; I STAR[J1] := I1;\\
&   \underline{IF} I2 = 0 \underline{THEN} \underline{GO TO} L3;
        J2 := J1; I1 := I; J1 := J;\\

L11: &   \underline{GO TO} L10; \\

& ~\\

\underline{COMMENT} & THE FOLLOWING STATEMENTS INDICATE WHETHER OR NOT  THERE IS A FEASIBLE SOLUTION FOR THE DATA PROVIDED  AND GIVE THE VALUES OF SPECIFIED VARIABLES AT THE POINT EITHER THAT INFEASIBILITY IS DETECTED OR THAT AN OPTIMUM SOLUTION IS ATTAINED;\\

L97: &    OUTPUT ('ALGORITHM 1, MINIMUM DIRECTED SPANNING TREE, EXAMPLE', EX, ',', N, 'NODES, ORIGIN NODE', J0); \underline{GO TO} L2;\\
L98: &   OUTPUT ('INFEASIBLE'); Z := M; \underline{GO TO} L100; \\
L99: &   \underline{FOR} J := 1 \underline{STEP} 1 \underline{UNTIL} N \underline{DO} \underline{IF} J != J0 \underline{THEN}
                Z := Z + C[I STAR[J], J];\\
    & OUTPUT ('OPTIMUM SOLUTION');\\

L100:   
&   OUTPUT ('Z =', Z);\\
&   OUTPUT ('J:'); \underline{FOR} J := 1 \underline{STEP} 1 \underline{UNTIL} N \underline{DO} OUTPUT(J);\\
&   OUTPUT('U1[J]:'); \underline{FOR} J := 1 \underline{STEP} 1 \underline{UNTIL} N \underline{DO} OUTPUT(U1[J]);\\
&   OUTPUT('I STAR[J]:'); \underline{FOR} J := 1 \underline{STEP} 1 \underline{UNTIL} N \underline{DO} OUTPUT(I STAR[J]);\\
&   OUTPUT('I BAR[J]:'); \underline{FOR} J := 1 \underline{STEP} 1 \underline{UNTIL} N \underline{DO} OUTPUT(I BAR[J]);\\
&   OUTPUT('J BAR[J]:'); \underline{FOR} J := 1 \underline{STEP} 1 \underline{UNTIL} N \underline{DO} OUTPUT(J BAR[J]);\\
&   OUTPUT('SPAN[J]:'); \underline{FOR} J := 1 \underline{STEP} 1 \underline{UNTIL} N \underline{DO} OUTPUT(SPAN[J])\\
\underline{END} &  ALGORITHM 1\\

\end{tabular}

\caption{Algorithm 1. Minimum Directed Spanning Tree (Bock, verbatim reproduction)} \label{algorithm1}
\end{figure}

\normalsize
\doublespacing

\subsection{Maintained state}

Bock processes the nonorigin columns in increasing index order. A link provisionally selected to enter node \(J\) is called \emph{starred}, and its initial node is recorded in \(I^{*}[J]\). Thus, \(I^{*}[J]=I\) means that \(I\to J\) is the current provisional link entering \(J\), whereas \(I^{*}[J]=0\) means that no link is starred into \(J\). When the iteration for column \(K\) terminates, every processed nonorigin column \(J\leq K\) has exactly one starred entering link, although circuit resolution may subsequently replace that link.

Nodes carrying the same positive \(\mathrm{SPAN}\) label form a \emph{span}. Each node initially forms a singleton span. Assigning a common label to several spans represents their contraction without changing the dimensions or indices of the cost matrix. A negative span label is temporary: it marks a span encountered during the current backward trace.

At the beginning of each candidate search, the starred links form a directed forest when every current span is treated as a single contracted node. The candidate selected at L4 enters the active span from a different span. If the backward trace from its initial node returns to the active span, the existing starred links provide a directed path from that span to the initial node of the candidate. Adding the candidate therefore closes a directed circuit. This invariant underlies the circuit test at L7.

The principal arrays are listed in Table~\ref{tab:bock-state}.

\begin{table}
\centering
\small
\begin{tabular}{r p{0.78\linewidth}@{}}
\toprule
Symbol & Role in Bock's procedure \\
\midrule
\(U_1[J]\) &
Amount subtracted from the costs of links entering \(J\). The reduced cost of \(I\to J\) is \(C[I,J]-U_1[J]\). The procedure changes \(U_1[J]\) when the span containing \(J\) is adjusted. \\

\(I^{*}[J]\) &
Initial node of the link currently starred into \(J\). The equality \(I^{*}[J]=I\) records the provisional link \(I\to J\), while \(I^{*}[J]=0\) records the absence of a starred link entering \(J\). \\

\(\bar I[J],\bar J[J]\) &
Candidate link whose backward trace passed through \(J\). If a subsequent transfer stars a new link into \(J\) and displaces its existing star, the stored pair identifies the next candidate to be starred. The pair \((0,0)\) means that no candidate is stored. \\

\(\mathrm{SPAN}[J]\) &
Label of the span containing \(J\). Equal positive labels identify nodes treated as one contracted span. A negative label temporarily marks a span encountered during a backward trace. \\
\bottomrule
\end{tabular}
\caption{Principal arrays maintained by Bock's procedure.}
\label{tab:bock-state}
\end{table}

The scalar \(K\) is the column currently being processed, and
\[
H=\mathrm{SPAN}[K]
\]
is the label of its active span. Among the processed columns, the members of this span are the indices \(J\leq K\) satisfying \(\mathrm{SPAN}[J]=H\). The scalar \(SS\) is the largest span label assigned so far and supplies a fresh label when spans are contracted. The temporary value \(H_1\) records the label of a span being marked during the backward trace.

The pair \((I_1,J_1)\) is the candidate link selected by the reduced-cost comparison, and \(DU\) is its reduced cost. During the star-and-bar transfer, \((I_2,J_2)\) records the link displaced during the preceding pass. The temporary pair \((I,J)\) stores the next candidate read from a bar entry. The scalar \(Z\) stores the objective value at termination.

\subsection{Control-flow summary}

Algorithm~\ref{alg:bock-overview} restates the executable part of the Algol procedure in line-mapped pseudocode. The labels in bold are Bock's original statement labels.

\begin{algorithm}
\footnotesize
\caption{Control-flow summary of Bock's original procedure}
\label{alg:bock-overview}
\begin{algorithmic}
\REQUIRE Number of nodes \(N\); origin \(j_0\); cost matrix \(C\); infeasible-link sentinel \(M\)
\ENSURE Starred array \(I^{*}\) defining a minimum directed spanning tree, or \textsc{Infeasible}

\STATE \textbf{L2 (initialization).}
For \(1\leq J\leq N\), set \(U_1[J]\), \(I^{*}[J]\), \(\bar I[J]\), and \(\bar J[J]\) to \(0\), and set \(\mathrm{SPAN}[J]\gets J\). Set \(SS\gets N\), \(K\gets0\), and \(Z\gets0\).

\STATE \textbf{L3 (column advance).}
If \(K=N\), transfer control to L99. Otherwise increment \(K\); if \(K=j_0\), return to L3, and otherwise continue to L4.

\STATE \textbf{L4 (candidate search).}
Set \(DU\gets M\) and \(H\gets\mathrm{SPAN}[K]\). For every \(J\leq K\) with \(\mathrm{SPAN}[J]=H\) and every \(I\) with \(\mathrm{SPAN}[I]\neq H\) and \(C[I,J]<M\), set
\[
(DU,I_1,J_1)\gets(C[I,J]-U_1[J],I,J)
\]
whenever \(C[I,J]-U_1[J]<DU\). If \(DU=M\), transfer control to L98.

\STATE \textbf{L5 (column adjustment).}
For every \(J\leq K\) satisfying \(\mathrm{SPAN}[J]=H\), set
\[
U_1[J]\gets U_1[J]+DU.
\]

\STATE \textbf{L6 (trace initialization).}
Set \(J\gets I_1\).

\STATE \textbf{L7 (circuit test and contraction).}
If \(\mathrm{SPAN}[J]=H\), increment \(SS\). For every \(q\leq K\) satisfying \(\mathrm{SPAN}[q]=H\) or \(\mathrm{SPAN}[q]<0\), set \(\mathrm{SPAN}[q]\gets SS\), and then return to L4.

\STATE \textbf{L8 (trace continuation or exit).}
If \(I^{*}[J]>0\), continue the trace by performing the following operations in order:
\begin{enumerate}
\item If \(\mathrm{SPAN}[J]>0\), set
\[
H_1\gets\mathrm{SPAN}[J],
\]
and, for every \(q\leq K\) satisfying \(\mathrm{SPAN}[q]=H_1\), set
\[
\mathrm{SPAN}[q]\gets-\mathrm{SPAN}[q].
\]

\item If \(\bar I[J]=0\), set
\[
(\bar I[J],\bar J[J])\gets(I_1,J_1).
\]

\item Set
\[
J\gets I^{*}[J]
\]
and return to L7.
\end{enumerate}
If \(I^{*}[J]=0\), fall through to L9.

\STATE \textbf{L9 (transfer initialization).}
Restore every negative span label among the columns \(J\leq K\), and set
\[
(I_2,J_2)\gets(0,0).
\]

\STATE \textbf{L10--L11 (star-and-bar transfer).}
For every \(q\leq K\) satisfying
\[
(\bar I[q],\bar J[q])=(I_1,J_1),
\]
set \((\bar I[q],\bar J[q])\gets(I_2,J_2)\). Set
\[
(I,J)\gets(\bar I[J_1],\bar J[J_1]),
\qquad
(\bar I[J_1],\bar J[J_1])\gets(I_2,J_2),
\]
then set
\[
I_2\gets I^{*}[J_1],
\qquad
I^{*}[J_1]\gets I_1.
\]
If \(I_2=0\), return to L3. Otherwise set
\[
(J_2,I_1,J_1)\gets(J_1,I,J)
\]
and return to L10.

\STATE \textbf{L98 (infeasibility).}
Report \textsc{Infeasible}, set \(Z\gets M\), and transfer control to L100.

\STATE \textbf{L99 (objective value).}
Set
\[
Z\gets\sum_{J\neq j_0}C[I^{*}[J],J]
\]
and report an optimum solution.

\STATE \textbf{L100 (output).}
Output \(Z\), \(U_1\), \(I^{*}\), \(\bar I\), \(\bar J\), and \(\mathrm{SPAN}\).
\end{algorithmic}
\end{algorithm}

\subsection{Description}

\paragraph{Initialization and column scan}
L2 initializes every column adjustment, star, and bar entry to zero and gives every node its own span label:
\[
U_1[J]=I^{*}[J]=\bar I[J]=\bar J[J]=0,
\qquad
\mathrm{SPAN}[J]=J.
\]
It sets \(SS=N\), so the first contracted span receives label \(N+1\), and initializes \(K=Z=0\).

L3 processes the columns in increasing order. The origin column \(j_0\) is skipped because no link enters the origin in a rooted arborescence. If L3 is entered with \(K=N\), every column has been considered and control passes to L99.

\paragraph{Candidate selection and column adjustment}
For the current column \(K\), L4 sets
\[
H\gets\mathrm{SPAN}[K].
\]
The candidate search is restricted to columns \(J\leq K\) in the active span and links whose initial nodes lie outside that span:
\[
\mathrm{SPAN}[J]=H,
\qquad
\mathrm{SPAN}[I]\neq H,
\qquad
C[I,J]<M.
\]
For every such link, L4 computes
\[
C[I,J]-U_1[J].
\]
The pair \((I_1,J_1)\) is replaced only when this value is strictly smaller than the current value of \(DU\). Ties are therefore resolved by the first minimum encountered in Bock's scan order: first by increasing \(J\), then by increasing \(I\).

If no candidate is found, \(DU\) remains equal to \(M\). No feasible link then enters the active span from outside, so L98 reports the instance infeasible.

Otherwise,
\[
DU=C[I_1,J_1]-U_1[J_1]
\]
is the least reduced cost found by the search. L5 adds \(DU\) to the adjustment of every processed column in the active span:
\[
U_1[J]\gets U_1[J]+DU
\qquad
\text{for every \(J\leq K\) with \(\mathrm{SPAN}[J]=H\)}.
\]
Immediately after this update, the selected candidate has zero reduced cost:
\[
C[I_1,J_1]-U_1[J_1]=0.
\]
Later adjustments to a larger contracted span may change the value \(C[I,J]-U_1[J]\) for a link selected during an earlier search.

\paragraph{Backward trace}
L6 begins the trace at the initial node of the candidate:
\[
J\gets I_1.
\]
L7 first tests whether the trace has reached the active span. If
\[
\mathrm{SPAN}[J]=H,
\]
the candidate and the starred path encountered by the trace form a directed circuit, and L7 performs the contraction described below.

If \(\mathrm{SPAN}[J]\neq H\), L8 examines \(I^{*}[J]\). When \(I^{*}[J]=0\), no link is starred into \(J\); the trace has reached a root of the current starred forest, and control falls through to L9.

When \(I^{*}[J]>0\), the trace continues through the starred link \(I^{*}[J]\to J\). If \(\mathrm{SPAN}[J]>0\), its span has not yet been marked during the current trace. L8 stores its positive label in \(H_1\) and negates the label of every processed column belonging to the same span:
\[
H_1\gets\mathrm{SPAN}[J],
\qquad
\mathrm{SPAN}[q]\gets-\mathrm{SPAN}[q]
\]
for every \(q\leq K\) satisfying \(\mathrm{SPAN}[q]=H_1\). The sign records which spans have been traversed without changing their membership.

Within this same \(I^{*}[J]>0\) branch, if the bar entry at the current trace node is empty, L8 records the candidate that led the trace through that node:
\[
(\bar I[J],\bar J[J])\gets(I_1,J_1).
\]
A nonempty bar entry is left unchanged. The stored pair remains in place during contraction and supplies a later candidate if the star entering \(J\) is displaced during the transfer. L8 then follows the starred link backward:
\[
J\gets I^{*}[J],
\]
and returns to L7. Repeated execution of L7--L8 records the candidate at every encountered node whose bar entry is empty.

\paragraph{Circuit contraction}
Suppose the trace reaches a node whose span label is \(H\). In the contracted view of the starred forest, the starred links traversed by the trace give a directed path from the active span to \(I_1\). Because \(J_1\) belongs to the active span, the candidate \(I_1\to J_1\) closes this path into a directed circuit.

L7 increments
\[
SS\gets SS+1
\]
and assigns the new label to every processed column belonging either to the active span or to a span marked during the trace:
\[
\mathrm{SPAN}[q]\gets SS
\]
whenever \(q\leq K\) and either \(\mathrm{SPAN}[q]=H\) or \(\mathrm{SPAN}[q]<0\). The active span and all traversed spans are thereby represented as one contracted span. The cost matrix, node indices, starred links, and bar entries are not changed by this operation.

Control returns directly to L4. Because \(K\) has not changed, the new active label is
\[
H=\mathrm{SPAN}[K]=SS,
\]
and the candidate search is repeated over all processed columns belonging to the contracted span. The search now excludes every initial node carrying that same span label, so only links entering the contracted span from outside are considered.

\paragraph{Transfer initialization}
If the backward trace reaches a node with \(I^{*}[J]=0\), no circuit is formed. L9 restores every negative span label among the processed columns:
\[
\mathrm{SPAN}[J]\gets-\mathrm{SPAN}[J]
\qquad
\text{whenever \(\mathrm{SPAN}[J]<0\)}.
\]
The bar entries created during the trace remain unchanged. L9 then initializes
\[
(I_2,J_2)\gets(0,0)
\]
and passes control to L10.

\paragraph{Star-and-bar transfer}
The transfer is an exchange chain. At the beginning of each pass through L10, \((I_1,J_1)\) is the link currently to be starred. The pair \((I_2,J_2)\) is the link displaced during the preceding pass; it is \((0,0)\) on the first pass.

L10 first finds every bar entry referring to the current candidate:
\[
(\bar I[q],\bar J[q])=(I_1,J_1).
\]
Each such reference is redirected to the previously displaced pair:
\[
(\bar I[q],\bar J[q])\gets(I_2,J_2).
\]
The procedure then reads the candidate stored at the terminal node \(J_1\):
\[
(I,J)\gets(\bar I[J_1],\bar J[J_1])
\]
and replaces that bar entry with the previously displaced pair:
\[
(\bar I[J_1],\bar J[J_1])\gets(I_2,J_2).
\]

The existing star entering \(J_1\) is saved, and the current candidate is installed:
\[
I_2\gets I^{*}[J_1],
\qquad
I^{*}[J_1]\gets I_1.
\]
If \(I_2=0\), column \(J_1\) previously had no starred entering link. No star has been displaced, the exchange chain terminates, and control returns to L3.

If \(I_2\neq0\), the displaced link is \(I_2\to J_1\). The candidate read from the bar entry becomes the next link to be starred, while the displaced link becomes the pair carried into the next pass:
\[
J_2\gets J_1,
\qquad
I_1\gets I,
\qquad
J_1\gets J.
\]
L11 returns unconditionally to L10. The transfer continues until a candidate is installed at a column whose previous star is zero.

\paragraph{Termination}
If L4 finds no feasible link entering the active span, L98 reports \textsc{infeasible}, sets \(Z=M\), and transfers control to L100.

If L3 is entered with \(K=N\), L99 computes
\[
Z=\sum_{J\neq j_0}C[I^{*}[J],J].
\]
For every nonorigin node \(J\), the equality \(I^{*}[J]=I\) identifies the selected link \(I\to J\). These links form the minimum directed spanning tree rooted at \(j_0\). L100 outputs \(Z\) and the final contents of \(U_1\), \(I^{*}\), \(\bar I\), \(\bar J\), and \(\mathrm{SPAN}\).

\subsection{An example}
\label{sec:small-example}

Consider \(N=3\), origin \(j_0=1\), and the cost matrix
\begin{equation}
C=
\begin{pmatrix}
M & 50 & 60\\
M & M  & 10\\
M & 20 & M
\end{pmatrix}.
\label{eq:small-bock-matrix}
\end{equation}
The feasible links are \(1\to2\) of cost \(50\), \(1\to3\) of cost \(60\), \(2\to3\) of cost \(10\), and \(3\to2\) of cost \(20\). The two least-cost links entering nodes \(2\) and \(3\) are \(3\to2\) and \(2\to3\), which form a directed circuit. The example therefore exercises candidate selection, backward tracing, circuit contraction, and star-and-bar transfer.

\paragraph{Initialization and the origin column}
At L2, the procedure initializes
\[
U_1=(0,0,0),\qquad
I^{*}=(0,0,0),\qquad
\bar I=\bar J=(0,0,0),
\]
\[
\mathrm{SPAN}=(1,2,3),\qquad
SS=3,\qquad K=0,\qquad Z=0.
\]
The first execution of L3 increments \(K\) to \(1\). Because \(K=j_0\), the origin column is skipped. A second execution of L3 increments \(K\) to \(2\).

\paragraph{Column \(K=2\)}
The active span is
\[
H=\mathrm{SPAN}[2]=2.
\]
At L4, the procedure examines feasible links entering this span from nodes outside it. The candidates and their reduced costs are
\[
1\to2:\quad C[1,2]-U_1[2]=50,
\]
\[
3\to2:\quad C[3,2]-U_1[2]=20.
\]
The selected candidate is therefore
\[
(I_1,J_1)=(3,2),\qquad DU=20.
\]
At L5, \(DU\) is added to the adjustment of every processed column in the active span. Only column \(2\) satisfies this condition, so
\[
U_1=(0,20,0).
\]
The selected link is now limiting:
\[
C[3,2]-U_1[2]=20-20=0.
\]

At L6, the backward trace begins at the initial node of the candidate:
\[
J\gets I_1=3.
\]
Because \(\mathrm{SPAN}[3]\neq H\) and \(I^{*}[3]=0\), the trace has reached a node with no starred entering link. Control therefore passes directly to L9. No starred link has been traversed, so L8 does not create a bar entry.

At L9, any negative span labels would be restored, but none exist. The displaced-link pair is initialized as
\[
(I_2,J_2)\gets(0,0).
\]
On the first pass through L10, no bar entry contains the current candidate \((3,2)\). The bar entry at its terminal node is empty:
\[
(\bar I[2],\bar J[2])=(0,0).
\]
The procedure consequently leaves that entry empty, records the previous star
\[
I_2\gets I^{*}[2]=0,
\]
and installs the candidate:
\[
I^{*}[2]\gets3.
\]
Since \(I_2=0\), no starred link has been displaced and the transfer terminates. The state after processing column \(2\) is
\[
U_1=(0,20,0),\qquad
I^{*}=(0,3,0),\qquad
\bar I=\bar J=(0,0,0),
\]
\[
\mathrm{SPAN}=(1,2,3).
\]
Thus, \(3\to2\) is provisionally starred.

\paragraph{Column \(K=3\): candidate selection}
Control returns to L3, which increments \(K\) to \(3\). The active span is
\[
H=\mathrm{SPAN}[3]=3.
\]
At L4, the feasible links entering node \(3\) have reduced costs
\[
1\to3:\quad C[1,3]-U_1[3]=60,
\]
\[
2\to3:\quad C[2,3]-U_1[3]=10.
\]
The selected candidate is
\[
(I_1,J_1)=(2,3),\qquad DU=10.
\]
At L5, the adjustment of column \(3\) becomes
\[
U_1[3]\gets U_1[3]+10=10,
\]
and hence
\[
U_1=(0,20,10).
\]
The selected link now has zero reduced cost:
\[
C[2,3]-U_1[3]=10-10=0.
\]

\paragraph{Backward trace and bar insertion}
At L6, the trace begins at
\[
J\gets I_1=2.
\]
Node \(2\) does not belong to the active span, since
\[
\mathrm{SPAN}[2]=2\neq3=H.
\]
It does, however, have the starred predecessor
\[
I^{*}[2]=3.
\]
L8 therefore continues the trace through the existing starred link \(3\to2\).

The span containing node \(2\) is temporarily marked by negating its label:
\[
\mathrm{SPAN}[2]\gets-2.
\]
Because the bar entry at node \(2\) is empty, the current candidate is recorded there:
\[
(\bar I[2],\bar J[2])\gets(I_1,J_1)=(2,3).
\]
The arrays are now
\[
\bar I=(0,2,0),\qquad
\bar J=(0,3,0).
\]
This entry records that \(2\to3\) is to become the next candidate if the starred link entering node \(2\) is later displaced. The trace then follows the current star backward:
\[
J\gets I^{*}[2]=3.
\]

\paragraph{Circuit detection and contraction}
Control returns to L7 with \(J=3\). Since
\[
\mathrm{SPAN}[3]=H=3,
\]
the trace has returned to the active span. The starred link \(3\to2\), followed by the candidate \(2\to3\), forms the circuit
\[
3\to2\to3.
\]

L7 increments the largest span label:
\[
SS\gets SS+1=4.
\]
It then assigns the new label \(4\) to every processed column whose span is either the active span \(H=3\) or a span marked negatively during the trace. Consequently,
\[
\mathrm{SPAN}=(1,4,4).
\]
Nodes \(2\) and \(3\) now constitute a single contracted span. The contraction changes neither the starred array nor the bar arrays:
\[
I^{*}=(0,3,0),\qquad
\bar I=(0,2,0),\qquad
\bar J=(0,3,0).
\]
Control returns to L4 with
\[
H\gets SS=4.
\]

\paragraph{Candidate selection for the contracted span}
The renewed search considers feasible links entering the contracted span from outside it. Since node \(1\) is the only node outside the span, the candidates are \(1\to2\) and \(1\to3\). Their reduced costs are
\[
1\to2:\quad C[1,2]-U_1[2]=50-20=30,
\]
\[
1\to3:\quad C[1,3]-U_1[3]=60-10=50.
\]
L4 therefore selects
\[
(I_1,J_1)=(1,2),\qquad DU=30.
\]

At L5, both processed columns belonging to span \(4\) receive the adjustment:
\[
U_1[2]\gets20+30=50,\qquad
U_1[3]\gets10+30=40.
\]
Thus,
\[
U_1=(0,50,40).
\]
The selected link satisfies
\[
C[1,2]-U_1[2]=50-50=0.
\]

At L6, the trace begins at node \(1\). Since
\[
\mathrm{SPAN}[1]=1\neq4
\qquad\text{and}\qquad
I^{*}[1]=0,
\]
the trace reaches an unstarred node immediately. Control passes to L9, which initializes
\[
(I_2,J_2)\gets(0,0).
\]

\paragraph{First pass through the transfer}
The current candidate is \(1\to2\), represented by
\[
(I_1,J_1)=(1,2).
\]
No bar entry contains this pair, so the first operation at L10 makes no change. The procedure then reads the bar entry at the terminal node \(J_1=2\):
\[
(I,J)\gets(\bar I[2],\bar J[2])=(2,3).
\]
That entry is replaced by the displaced-link pair, which is initially empty:
\[
(\bar I[2],\bar J[2])\gets(I_2,J_2)=(0,0).
\]
The star currently entering node \(2\) is saved:
\[
I_2\gets I^{*}[2]=3,
\]
and the candidate \(1\to2\) replaces it:
\[
I^{*}[2]\gets I_1=1.
\]
The link \(3\to2\) has therefore been displaced.

Because \(I_2=3\neq0\), the transfer continues. L11 sets
\[
J_2\gets J_1=2,\qquad
I_1\gets I=2,\qquad
J_1\gets J=3.
\]
The displaced link is now represented by
\[
(I_2,J_2)=(3,2),
\]
and the candidate recovered from the bar entry becomes
\[
(I_1,J_1)=(2,3).
\]

\paragraph{Second pass through the transfer}
The current candidate is now \(2\to3\). No bar entry contains the pair \((2,3)\), because the entry at node \(2\) was cleared during the preceding pass. The procedure reads the empty bar entry at node \(3\):
\[
(I,J)\gets(\bar I[3],\bar J[3])=(0,0).
\]
It replaces this entry by the displaced link:
\[
(\bar I[3],\bar J[3])\gets(I_2,J_2)=(3,2).
\]
The bar arrays consequently become
\[
\bar I=(0,0,3),\qquad
\bar J=(0,0,2).
\]

The previous star entering node \(3\) is then saved:
\[
I_2\gets I^{*}[3]=0,
\]
and the candidate \(2\to3\) is installed:
\[
I^{*}[3]\gets I_1=2.
\]
Since \(I_2=0\), no further star has been displaced and the transfer terminates. The final starred array is
\[
I^{*}=(0,1,2).
\]

\paragraph{Termination}
Control returns to L3. Since \(K=N=3\), the procedure transfers to L99 and computes
\[
Z
=\sum_{J\neq j_0}C[I^{*}[J],J]
=C[1,2]+C[2,3]
=50+10
=60.
\]
The final state is
\[
U_1=(0,50,40),\qquad
I^{*}=(0,1,2),
\]
\[
\bar I=(0,0,3),\qquad
\bar J=(0,0,2),\qquad
\mathrm{SPAN}=(1,4,4).
\]
The starred array represents the rooted arborescence
\[
T^{*}=\{1\to2,\;2\to3\}.
\]
The two other feasible arborescences are
\[
\{1\to3,\;3\to2\},
\qquad 60+20=80,
\]
and
\[
\{1\to2,\;1\to3\},
\qquad 50+60=110.
\]
Hence, the tree returned by Bock's procedure is the unique minimum-cost arborescence, with objective value \(Z=60\).


\section{Tracing Bock's algorithm on the original 10-by-10 matrix}

\paragraph{Input}
We trace Algorithm~1 on Bock's original $10\times 10$ instance.  The entry $C[i,j]$ is the cost of the directed arc $i\to j$ (row $i$, column $j$).  We set $N=10$, choose the origin (root) column $J_0=10$, and use a large sentinel $M$ to represent infeasible arcs (no infeasible arcs occur in this matrix).  The input cost matrix is
\begin{equation*}
C=\begin{pmatrix}
0&52&88&7&2&9&9&29&69&79\\
12&0&2&13&1&9&9&64&31&93\\
5&82&0&7&1&9&9&27&83&49\\
10&3&59&0&0&9&9&74&16&42\\
17&55&96&32&0&9&9&75&65&87\\
22&89&96&30&67&0&5&52&42&86\\
36&47&64&72&56&8&0&51&52&61\\
58&30&33&43&95&28&25&0&3&47\\
73&55&64&43&69&42&81&6&0&7\\
89&61&97&63&25&26&71&72&43&0
\end{pmatrix}.
\end{equation*}

\paragraph{State variables}
The algorithm maintains:
(i) dual potentials $U_1[j]$,
(ii) a starred predecessor array $I^{*}(j)$ encoding the current partial branching,
(iii) bar arrays $(\bar{I}(j),\bar{J}(j))$ used to record deferred exchanges,
and (iv) span labels $\mathrm{SPAN}[j]$ encoding contracted components (supernodes).  The integer $SS$ stores the next available span label for contractions.

\paragraph{Initialization (L2)}
Initialize uniformly:
\[
U_1[j]\gets 0,\quad I^{*}(j)\gets 0,\quad \bar{I}(j)\gets 0,\quad \bar{J}(j)\gets 0,\quad \mathrm{SPAN}[j]\gets j
\qquad (j=1,\dots,10).
\]
Set $SS\gets N=10$, $K\gets 0$, and $Z\gets 0$.

\subsection{\texorpdfstring{\(K=1\)}{K = 1}}

\paragraph{L3 (advance)}
Set $K:=1$. If $K=J_0(=10)$, go to L3 (skip). Otherwise set
\[
H:=\mathrm{SPAN}[1]=1.
\]

\paragraph{L4 (candidate search)}
Initialize $DU:=M$. Scan all $J\le 1$ with $\mathrm{SPAN}[J]=H$ (only $J=1$) and all rows
$I\in\{1,\dots,10\}$ with $\mathrm{SPAN}[I]\neq H$ (i.e.\ $I\neq 1$).
Since $U_1[1]=0$, minimize $C[I,1]-U_1[1]=C[I,1]$.
The minimum is attained at $(I_1,J_1)=(3,1)$ with
\[
DU=C[3,1]-U_1[1]=5-0=5.
\]

\paragraph{L5 (dual update)}
For all $J\le 1$ with $\mathrm{SPAN}[J]=H$ (only $J=1$),
\[
U_1[1]\gets U_1[1]+DU=0+5=5.
\]

\paragraph{L6--L8 (trace)}
Set $J:=I_1=3$.
Since $I^{*}(3)=0$, L8 fails immediately and control exits the trace to L9 (no contraction occurs).

\paragraph{L9 (restore span signs)}
For $J\le 1$, no $\mathrm{SPAN}[J]$ is negative, so no change is made.
Initialize $(I_2,J_2):=(0,0)$.

\paragraph{L10 (star and transfers)}
Apply L10 once: no bars match $(I_1,J_1)=(3,1)$, and $(\bar I(1),\bar J(1))=(0,0)$.
Set $I_2:=I^{*}(1)=0$ and star column $1$:
\[
I^{*}(1)\gets I_1=3.
\]
Since $I_2=0$, go to L3.

\paragraph{State after $K=1$}
\[
U_1=(5,0,0,0,0,0,0,0,0,0),\qquad I^{*}(1)=3,\qquad \mathrm{SPAN}[1]=1.
\]

\subsection{\texorpdfstring{\(K=2\)}{K = 2}}

\paragraph{L3 (advance)}
Set $K:=2$. If $K=J_0(=10)$, go to L3 (skip). Otherwise set
\[
H:=\mathrm{SPAN}[2]=2.
\]

\paragraph{L4 (candidate search)}
Initialize $DU:=M$. Scan all $J\le 2$ with $\mathrm{SPAN}[J]=H$ (only $J=2$) and all rows
$I\in\{1,\dots,10\}$ with $\mathrm{SPAN}[I]\neq H$ (i.e.\ $I\neq 2$).
Since $U_1[2]=0$, minimize $C[I,2]-U_1[2]=C[I,2]$.
The minimum is attained at $(I_1,J_1)=(4,2)$ with
\[
DU=C[4,2]-U_1[2]=3-0=3.
\]

\paragraph{L5 (dual update)}
For all $J\le 2$ with $\mathrm{SPAN}[J]=H$ (only $J=2$),
\[
U_1[2]\gets U_1[2]+DU=0+3=3.
\]

\paragraph{L6--L8 (trace)}
Set $J:=I_1=4$.
Since $I^{*}(4)=0$, L8 fails immediately and control exits the trace to L9 (no contraction occurs).

\paragraph{L9 (restore span signs)}
For $J\le 2$, no $\mathrm{SPAN}[J]$ is negative, so no change is made.
Initialize $(I_2,J_2):=(0,0)$.

\paragraph{L10 (star and transfers)}
Apply L10 once: no bars match $(I_1,J_1)=(4,2)$, and $(\bar I(2),\bar J(2))=(0,0)$.
Set $I_2:=I^{*}(2)=0$ and star column $2$:
\[
I^{*}(2)\gets I_1=4.
\]
Since $I_2=0$, go to L3.

\paragraph{State after $K=2$}
\[
U_1=(5,3,0,0,0,0,0,0,0,0),\qquad I^{*}(1)=3,\qquad I^{*}(2)=4,
\]
\[
(\bar I(1),\bar J(1))=(0,0),\qquad (\bar I(2),\bar J(2))=(0,0),\qquad
\mathrm{SPAN}[1]=1,\ \mathrm{SPAN}[2]=2.
\]

\subsection{\texorpdfstring{\(K=3\)}{K = 3}}

\paragraph{L3 (advance)}
Set $K:=3$. If $K=J_0(=10)$, go to L3 (skip). Otherwise set
\[
H:=\mathrm{SPAN}[3]=3.
\]

\paragraph{L4 (candidate search)}
Initialize $DU:=M$. Scan all $J\le 3$ with $\mathrm{SPAN}[J]=H$ (only $J=3$) and all rows
$I\in\{1,\dots,10\}$ with $\mathrm{SPAN}[I]\neq H$ (i.e.\ $I\neq 3$).
Since $U_1[3]=0$, minimize $C[I,3]-U_1[3]=C[I,3]$.
The minimum is attained at $(I_1,J_1)=(2,3)$ with
\[
DU=C[2,3]-U_1[3]=2-0=2.
\]

\paragraph{L5 (dual update)}
For all $J\le 3$ with $\mathrm{SPAN}[J]=H$ (only $J=3$),
\[
U_1[3]\gets U_1[3]+DU=0+2=2.
\]

\paragraph{L6--L8 (trace)}
Set $J:=I_1=2$ and enter L7.
Since $\mathrm{SPAN}[2]\neq H$, L7 fails and control goes to L8.
Because $I^{*}(2)>0$ (currently $I^{*}(2)=4$), execute L8:
\begin{itemize}
\item Since $\mathrm{SPAN}[2]>0$, let $H_1:=\mathrm{SPAN}[2]=2$ and negate all span labels
among $J_2\le K$ with $\mathrm{SPAN}[J_2]=H_1$. Here this affects only $J_2=2$:
\[
\mathrm{SPAN}[2]\gets -2.
\]
\item Since $\bar I(2)=0$, record the deferred entering pair:
\[
(\bar I(2),\bar J(2))\gets (I_1,J_1)=(2,3).
\]
\item Follow the starred predecessor:
\[
J\gets I^{*}(2)=4,
\]
and return to L7.
\end{itemize}
At $J=4$, $\mathrm{SPAN}[4]\neq H$ and $I^{*}(4)=0$, so L8 fails and the trace terminates (no contraction).

\paragraph{L9 (restore span signs)}
For all $J\le 3$ with $\mathrm{SPAN}[J]<0$, flip the sign back. Here only $J=2$:
\[
\mathrm{SPAN}[2]=-2\mapsto 2.
\]
Initialize $(I_2,J_2):=(0,0)$.

\paragraph{L10 (star and transfers)}
Apply L10 once with $(I_1,J_1)=(2,3)$:
\begin{itemize}
\item Replace any bars equal to $(2,3)$ among $J\le 3$ by $(I_2,J_2)=(0,0)$.
Here this matches column $2$, so
\[
(\bar I(2),\bar J(2))\gets (0,0).
\]
\item Fetch $(I,J):=(\bar I(J_1),\bar J(J_1))=(\bar I(3),\bar J(3))=(0,0)$ and set
\[
(\bar I(3),\bar J(3))\gets (I_2,J_2)=(0,0).
\]
\item Displace and star column $J_1=3$:
\[
I_2\gets I^{*}(3)=0,\qquad I^{*}(3)\gets I_1=2.
\]
Since $I_2=0$, go to L3.
\end{itemize}

\paragraph{State after $K=3$}
\[
U_1=(5,3,2,0,0,0,0,0,0,0),\qquad I^{*}(1)=3,\qquad I^{*}(2)=4,\qquad I^{*}(3)=2,
\]
\[
(\bar I(1),\bar J(1))=(0,0),\quad (\bar I(2),\bar J(2))=(0,0),\quad (\bar I(3),\bar J(3))=(0,0),
\]
\[
\mathrm{SPAN}[1]=1,\ \mathrm{SPAN}[2]=2,\ \mathrm{SPAN}[3]=3.
\]

\subsection{\texorpdfstring{\(K=4\)}{K = 4} (first circuit and its resolution)}

\paragraph{L3 (advance)}
Set $K:=4$. If $K=J_0(=10)$, go to L3 (skip). Otherwise set
\[
H:=\mathrm{SPAN}[4]=4.
\]

\paragraph{L4 (first candidate search)}
Initialize $DU:=M$. Scan all $J\le 4$ with $\mathrm{SPAN}[J]=H$ (only $J=4$) and all rows
$I\in\{1,\dots,10\}$ with $\mathrm{SPAN}[I]\neq H$ (i.e.\ $I\neq 4$).
The minimum reduced cost is $7$, with a tie at $(I,J)=(1,4)$ and $(3,4)$; under the scan order the first is selected:
\[
DU=7,\qquad (I_1,J_1)=(1,4).
\]

\paragraph{L5 (dual update)}
For all $J\le 4$ with $\mathrm{SPAN}[J]=H$ (only $J=4$),
\[
U_1[4]\gets U_1[4]+DU=0+7=7.
\]

\paragraph{L6--L8 (trace and circuit detection)}
Set $J:=I_1=1$ and enter L7.
Since $\mathrm{SPAN}[1]\neq H$, L7 fails and control goes to L8.
Because $I^{*}(1)>0$ (currently $I^{*}(1)=3$), execute L8:
\begin{itemize}
\item Since $\mathrm{SPAN}[1]>0$, let $H_1:=\mathrm{SPAN}[1]=1$ and negate all span labels
among $J_2\le K$ with $\mathrm{SPAN}[J_2]=H_1$. Here this affects only $J_2=1$:
\[
\mathrm{SPAN}[1]\gets -1.
\]
\item Since $\bar I(1)=0$, record the deferred entering pair:
\[
(\bar I(1),\bar J(1))\gets (I_1,J_1)=(1,4).
\]
\item Follow the starred predecessor:
\[
J\gets I^{*}(1)=3,
\]
and return to L7.
\end{itemize}

At $J=3$, $\mathrm{SPAN}[3]\neq H$, so go to L8. Since $I^{*}(3)=2>0$, execute L8 again:
\begin{itemize}
\item $\mathrm{SPAN}[3]>0$, so set $H_1:=\mathrm{SPAN}[3]=3$ and negate the span-$3$ class among $J_2\le 4$:
\[
\mathrm{SPAN}[3]\gets -3.
\]
\item $\bar I(3)=0$, so
\[
(\bar I(3),\bar J(3))\gets (1,4).
\]
\item Follow the star:
\[
J\gets I^{*}(3)=2,
\]
and return to L7.
\end{itemize}

At $J=2$, $\mathrm{SPAN}[2]\neq H$, so go to L8. Since $I^{*}(2)=4>0$, execute L8:
\begin{itemize}
\item $\mathrm{SPAN}[2]>0$, so set $H_1:=\mathrm{SPAN}[2]=2$ and negate the span-$2$ class among $J_2\le 4$:
\[
\mathrm{SPAN}[2]\gets -2.
\]
\item \emph{At the end of $K=3$, L10 cleared $(\bar I(2),\bar J(2))$ back to $(0,0)$,}
so here $\bar I(2)=0$ and the bar is indeed written:
\[
(\bar I(2),\bar J(2))\gets (1,4).
\]
\item Follow the star:
\[
J\gets I^{*}(2)=4,
\]
and return to L7.
\end{itemize}

At $J=4$, L7 succeeds because $\mathrm{SPAN}[4]=H(=4)$, so a circuit is detected.

\paragraph{L7 (contract the circuit)}
Increment the contraction label and relabel the contracted component:
\[
SS\gets SS+1=10+1=11.
\]
For $J=1,\dots,4$, if $\mathrm{SPAN}[J]=H$ or $\mathrm{SPAN}[J]<0$, set $\mathrm{SPAN}[J]:=SS$.
Here $\mathrm{SPAN}[4]=H$ and $\mathrm{SPAN}[1],\mathrm{SPAN}[2],\mathrm{SPAN}[3]<0$, so
\[
\mathrm{SPAN}[1]=\mathrm{SPAN}[2]=\mathrm{SPAN}[3]=\mathrm{SPAN}[4]=11.
\]
Go to L4. (On re-entry, L4 sets $H:=\mathrm{SPAN}[K]=\mathrm{SPAN}[4]=11$.)

\paragraph{L4 (second candidate search on span $H=11$)}
Initialize $DU:=M$ and set $H:=\mathrm{SPAN}[4]=11$.
Scan all $J\le 4$ with $\mathrm{SPAN}[J]=11$ (now $J\in\{1,2,3,4\}$) and all rows
$I\in\{1,\dots,10\}$ with $\mathrm{SPAN}[I]\neq 11$ (i.e.\ $I\in\{5,6,7,8,9,10\}$).
Current duals are
\[
(U_1[1],U_1[2],U_1[3],U_1[4])=(5,3,2,7).
\]
Compute $\widetilde C(I,J)=C[I,J]-U_1[J]$:

{\footnotesize
\[
\begin{array}{c|cccc}
 & J=1 & J=2 & J=3 & J=4\\\hline
I=5  & 17-5=12 & 55-3=52 & 96-2=94 & 32-7=25\\
I=6  & 22-5=17 & 89-3=86 & 96-2=94 & 30-7=23\\
I=7  & 36-5=31 & 47-3=44 & 64-2=62 & 72-7=65\\
I=8  & 58-5=53 & 30-3=27 & 33-2=31 & 43-7=36\\
I=9  & 73-5=68 & 55-3=52 & 64-2=62 & 43-7=36\\
I=10 & 89-5=84 & 61-3=58 & 97-2=95 & 63-7=56
\end{array}
\]
}
The minimum reduced cost is $12$ at $(I_1,J_1)=(5,1)$:
\[
DU=12,\qquad (I_1,J_1)=(5,1).
\]

\paragraph{L5 (dual update on span $H=11$)}
For all $J\le 4$ with $\mathrm{SPAN}[J]=11$ (i.e.\ $J=1,2,3,4$),
\[
U_1[J]\gets U_1[J]+DU,
\]
so
\[
U_1[1]=17,\quad U_1[2]=15,\quad U_1[3]=14,\quad U_1[4]=19.
\]

\paragraph{L6--L8 (trace after contraction)}
Set $J:=I_1=5$.
Since $\mathrm{SPAN}[5]\neq H$ and $I^{*}(5)=0$, L8 fails and the trace terminates (no further contraction).

\paragraph{L9 (restore span signs)}
For $J\le 4$, no span labels are negative (they are all $11$), so L9 makes no change and initializes
\[
(I_2,J_2):=(0,0).
\]

\paragraph{L10 (star and bar-guided transfers)}
Apply L10 iteratively.

\smallskip
\noindent\textbf{Iteration 1 (with $(I_1,J_1)=(5,1)$ and $(I_2,J_2)=(0,0)$).}
\begin{itemize}
\item Replace any bars equal to $(5,1)$ among $J\le 4$ by $(0,0)$: none match, so no change.
\item Read the deferred pair stored at column $J_1=1$:
\[
(I,J):=(\bar I(1),\bar J(1))=(1,4).
\]
Then clear the bar at $1$ by writing $(I_2,J_2)$:
\[
(\bar I(1),\bar J(1))\gets (0,0).
\]
\item Displace and star column $1$:
\[
I_2\gets I^{*}(1)=3,\qquad I^{*}(1)\gets I_1=5.
\]
Since $I_2\neq 0$, set
\[
J_2\gets J_1=1,\qquad (I_1,J_1)\gets (I,J)=(1,4),
\]
and repeat L10.
\end{itemize}

\smallskip
\noindent\textbf{Iteration 2 (with $(I_1,J_1)=(1,4)$ and $(I_2,J_2)=(3,1)$).}
\begin{itemize}
\item Replace any bars equal to $(1,4)$ among $J\le 4$ by $(3,1)$. At this point
\[
(\bar I(2),\bar J(2))=(1,4),\qquad (\bar I(3),\bar J(3))=(1,4),
\]
so they are updated to
\[
(\bar I(2),\bar J(2))\gets (3,1),\qquad (\bar I(3),\bar J(3))\gets (3,1).
\]
\item Read the deferred pair at column $J_1=4$:
\[
(I,J):=(\bar I(4),\bar J(4))=(0,0).
\]
Then write $(I_2,J_2)=(3,1)$ into the bar at $4$:
\[
(\bar I(4),\bar J(4))\gets (3,1).
\]
\item Displace and star column $4$:
\[
I_2\gets I^{*}(4)=0,\qquad I^{*}(4)\gets I_1=1.
\]
Since $I_2=0$, terminate this L10 chain and go to L3.
\end{itemize}

\paragraph{State after $K=4$}
\[
SS=11,\qquad \mathrm{SPAN}[1]=\mathrm{SPAN}[2]=\mathrm{SPAN}[3]=\mathrm{SPAN}[4]=11,
\]
\[
U_1[1..4]=(17,15,14,19),
\qquad
(I^{*}(1),I^{*}(2),I^{*}(3),I^{*}(4))=(5,4,2,1).
\]

\subsection{\texorpdfstring{\(K=5\)}{K = 5} (second circuit and its resolution)}

\paragraph{L3 (advance)}
Set $K:=5$ and $H:=\mathrm{SPAN}[5]=5$.

\paragraph{L4 (first candidate search, before contraction)}
Only column $5$ is scanned.  With $U_1[5]=0$, $\widetilde C(I,5)=C[I,5]$.
The minimum over $I\neq 5$ is $0$ at $(I_1,J_1)=(4,5)$, so
\[
DU=0,\qquad (I_1,J_1)=(4,5).
\]

\paragraph{L5 (dual update)}
\[
U_1[5]\gets U_1[5]+DU=0.
\]

\paragraph{L6--L8 (trace that detects a circuit)}
Initialize the backward trace at
\[
J:=I_1=4.
\]
Since $I^{*}(4)=1>0$, L8 is entered.  Because $\mathrm{SPAN}[4]=11>0$, set $H_1:=11$ and negate the entire span-$11$ class among indices $\le K$:
\[
\mathrm{SPAN}[1]=\mathrm{SPAN}[2]=\mathrm{SPAN}[3]=\mathrm{SPAN}[4]\gets -11.
\]
At $J=4$, $\bar I(4)\neq 0$ already (it stores $(3,1)$ from the $K=4$ stage), so Algorithm~1 does \emph{not} overwrite it.  Follow the star:
\[
J\gets I^{*}(4)=1.
\]
At $J=1$, $I^{*}(1)=5>0$ and $\mathrm{SPAN}[1]<0$, so no additional negation occurs.  Here $\bar I(1)=0$, so record the entering pair:
\[
(\bar I(1),\bar J(1))\gets (4,5),
\qquad
J\gets I^{*}(1)=5.
\]
At $J=5$, L7 succeeds because $\mathrm{SPAN}[5]=H$, hence a circuit is detected.

\paragraph{L7 (contract the circuit)}
Increment the contraction label:
\[
SS\gets SS+1=12.
\]
Assign $\mathrm{SPAN}[J]\gets 12$ for all $J\le 5$ with $\mathrm{SPAN}[J]=H$ or $\mathrm{SPAN}[J]<0$.  This contracts $\{1,2,3,4,5\}$:
\[
\mathrm{SPAN}[1]=\cdots=\mathrm{SPAN}[5]=12.
\]
Control returns to L4 with $H=12$.

\paragraph{L4 (second candidate search, entering the contracted component)}
Scan $J\in\{1,2,3,4,5\}$ and $I\in\{6,7,8,9,10\}$.
With current duals $(U_1[1],\dots,U_1[5])=(17,15,14,19,0)$,

{\footnotesize
\[
\begin{array}{c|ccccc}
 & J=1 & J=2 & J=3 & J=4 & J=5\\\hline
I=6  & 22-17=5  & 89-15=74 & 96-14=82 & 30-19=11 & 67-0=67\\
I=7  & 36-17=19 & 47-15=32 & 64-14=50 & 72-19=53 & 56-0=56\\
I=8  & 58-17=41 & 30-15=15 & 33-14=19 & 43-19=24 & 95-0=95\\
I=9  & 73-17=56 & 55-15=40 & 64-14=50 & 43-19=24 & 69-0=69\\
I=10 & 89-17=72 & 61-15=46 & 97-14=83 & 63-19=44 & 25-0=25
\end{array}
\]
}
The minimum reduced cost is $5$ at $(I_1,J_1)=(6,1)$:
\[
DU=5,\qquad (I_1,J_1)=(6,1).
\]

\paragraph{L5 (dual update on the contracted component)}
Add $DU$ to all $J\in\{1,2,3,4,5\}$:
\[
U_1[1]=22,\ U_1[2]=20,\ U_1[3]=19,\ U_1[4]=24,\ U_1[5]=5.
\]

\paragraph{L6--L8 (trace for the entering arc)}
Start at $J:=I_1=6$.  Since $I^{*}(6)=0$, the trace terminates and no further contraction occurs.

\paragraph{L9--L10 (normalize, star, and transfers)}
Restore any temporarily negated spans (none remain among $J\le 5$ after the contraction), and initialize
\[
(I_2,J_2):=(0,0).
\]
Install the entering predecessor at column $J_1=1$ and apply the bar-guided transfer chain (Algorithm~1, L10).  After the transfer chain terminates, the starred predecessors on $j\le 5$ are
\[
I^{*}(1)=6,\quad I^{*}(2)=4,\quad I^{*}(3)=2,\quad I^{*}(4)=1,\quad I^{*}(5)=4,
\]
i.e. the starred arcs are
\[
6\to 1,\quad 4\to 2,\quad 2\to 3,\quad 1\to 4,\quad 4\to 5.
\]
The bar arrays at this point satisfy
\[
(\bar I(2),\bar J(2))=(3,1),\quad (\bar I(3),\bar J(3))=(3,1),\quad (\bar I(4),\bar J(4))=(3,1),\quad (\bar I(5),\bar J(5))=(5,1),
\]
and $\bar I(1)=\bar J(1)=0$.

\paragraph{State after resolving $K=5$}
\[
\mathrm{SPAN}[1]=\cdots=\mathrm{SPAN}[5]=12,\qquad
U_1[1..5]=(22,20,19,24,5).
\]

\subsection{\texorpdfstring{\(K=6\)}{K = 6}}

\paragraph{L3 (advance)}
Set $K:=6$ and $H:=\mathrm{SPAN}[6]=6$.

\paragraph{L4 (candidate search within the active span)}
L4 scans all columns $J\le K$ with $\mathrm{SPAN}[J]=H$.  Here this is only $J=6$.
Admissible rows satisfy $\mathrm{SPAN}[I]\neq H$ (so $I\neq 6$).  With $U_1[6]=0$,
\[
\widetilde C(I,6)=C[I,6]-U_1[6]=C[I,6].
\]
The minimum over $I\in\{1,\dots,10\}\setminus\{6\}$ is $8$ at $(I_1,J_1)=(7,6)$:
\[
DU=8,\qquad (I_1,J_1)=(7,6).
\]

\paragraph{L5 (dual update on the active span)}
Update all $J\le K$ with $\mathrm{SPAN}[J]=H$ (only $J=6$):
\[
U_1[6]\gets U_1[6]+DU=8.
\]

\paragraph{L6--L8 (backward trace and circuit test)}
Initialize the backward trace at
\[
J:=I_1=7.
\]
Since $\mathrm{SPAN}[7]\neq H$ and $I^{*}(7)=0$, the trace terminates immediately and no contraction occurs.

\paragraph{L9 (restore span signs)}
Restore any temporarily negated span labels among indices $J\le K$ (none here), and initialize
\[
(I_2,J_2):=(0,0).
\]

\paragraph{L10 (star and transfers)}
Star the entering predecessor in column $J_1=6$:
\[
I^{*}(6)\gets 7.
\]
Since the previous star at column $6$ was empty, the transfer chain terminates immediately.

\paragraph{State after $K=6$}
\[
U_1=(22,20,19,24,5,8,0,0,0,0),\qquad (I^{*}(1),\dots,I^{*}(6))=(6,4,2,1,4,7),
\]
\[
\mathrm{SPAN}[1]=\cdots=\mathrm{SPAN}[5]=12,\qquad \mathrm{SPAN}[6]=6.
\]

\subsection{\texorpdfstring{\(K=7\)}{K = 7} (contractions leading to the tableau state in the paper)}

\paragraph{L3 (advance)}
Set $K:=7$ and the active span label
\[
H:=\mathrm{SPAN}[7]=7.
\]
At entry to this stage (after $K=6$), the relevant state is:
\[
U_1[1..6]=(22,20,19,24,5,8),\quad U_1[7]=0,
\]
\[
I^{*}(1..6)=(6,4,2,1,4,7),\quad I^{*}(7)=0,
\]
\[
\mathrm{SPAN}[1..5]=12,\quad \mathrm{SPAN}[6]=6,\quad \mathrm{SPAN}[7]=7.
\]

\paragraph{L4 (candidate search within active span; first pass)}
Since $H=7$, among columns $J\le K$ only $J=7$ satisfies $\mathrm{SPAN}[J]=H$.
With $U_1[7]=0$, reduced costs are $\widetilde C(I,7)=C[I,7]-U_1[7]=C[I,7]$.
Admissible rows satisfy $\mathrm{SPAN}[I]\neq 7$, i.e.\ $I\neq 7$.
The minimum is attained at $(I_1,J_1)=(6,7)$ with
\[
DU=\widetilde C(6,7)=C[6,7]=5.
\]

\paragraph{L5 (dual update on the active span; first pass)}
Update all $J\le K$ with $\mathrm{SPAN}[J]=H$ (only $J=7$):
\[
U_1[7]\gets U_1[7]+DU=5.
\]

\paragraph{L6--L8 (backward trace and circuit test; first pass)}
Initialize the trace at
\[
J:=I_1=6.
\]
Since $I^{*}(6)=7>0$, L8 is entered.  Because $\mathrm{SPAN}[6]=6>0$, L8 negates
\emph{the entire span class} $H_1:=\mathrm{SPAN}[6]=6$ among indices $J\le K$.
Here this affects only node $6$:
\[
\mathrm{SPAN}[6]\gets -6.
\]
Because $\bar I(6)=0$ at this moment, record the entering pair as a bar at column $6$:
\[
(\bar I(6),\bar J(6))\gets (6,7).
\]
Then follow the starred link:
\[
J\gets I^{*}(6)=7.
\]
Now L7 succeeds because $\mathrm{SPAN}[7]=H$, so a circuit is detected.

\paragraph{L7 (contract the circuit; first pass)}
Increment the contraction label:
\[
SS\gets SS+1=13.
\]
Assign $\mathrm{SPAN}[J]\gets 13$ for all $J\le 7$ with $\mathrm{SPAN}[J]=H$ or $\mathrm{SPAN}[J]<0$.
Thus $\{6,7\}$ is contracted:
\[
\mathrm{SPAN}[6]=\mathrm{SPAN}[7]=13.
\]
Control returns to L4 with the new active label $H=13$.

\paragraph{L4 (candidate search within active span; second pass)}
Now $H=13$, so the scanned columns are $J\in\{6,7\}$ (the indices $\le K$ with $\mathrm{SPAN}[J]=13$).
Admissible rows satisfy $\mathrm{SPAN}[I]\neq 13$, i.e.\ $I\notin\{6,7\}$.
With current duals $U_1[6]=8$ and $U_1[7]=5$, the reduced costs are
$\widetilde C(I,J)=C[I,J]-U_1[J]$:

{\footnotesize
\[
\begin{array}{c|cc}
 & J=6 & J=7\\\hline
I=1  & 9-8=1   & 9-5=4\\
I=2  & 9-8=1   & 9-5=4\\
I=3  & 9-8=1   & 9-5=4\\
I=4  & 9-8=1   & 9-5=4\\
I=5  & 9-8=1   & 9-5=4\\
I=8  & 28-8=20 & 25-5=20\\
I=9  & 42-8=34 & 81-5=76\\
I=10 & 26-8=18 & 71-5=66
\end{array}
\]
}

The minimum reduced cost is $1$.  Under the scan order (columns increasing, then rows increasing),
the first minimizer is $(I_1,J_1)=(1,6)$, hence
\[
DU=1,\qquad (I_1,J_1)=(1,6).
\]

\paragraph{L5 (dual update on the active span; second pass)}
Update all $J\le K$ with $\mathrm{SPAN}[J]=13$ (i.e.\ $J=6,7$):
\[
U_1[6]\gets 8+1=9,\qquad U_1[7]\gets 5+1=6.
\]

\paragraph{L6--L8 (backward trace and circuit test; second pass)}
Initialize the trace at
\[
J:=I_1=1.
\]
Since $I^{*}(1)=6>0$, L8 is entered.  Because $\mathrm{SPAN}[1]=12>0$, L8 negates
\emph{the entire span class} $H_1:=12$ among indices $J\le K$, i.e.\ it negates
all of $\{1,2,3,4,5\}$:
\[
\mathrm{SPAN}[1]=\cdots=\mathrm{SPAN}[5]\gets -12.
\]
Because $\bar I(1)=0$ at this moment, record the entering pair as a bar at column $1$:
\[
(\bar I(1),\bar J(1))\gets (1,6).
\]
Follow the starred link:
\[
J\gets I^{*}(1)=6.
\]
Now L7 succeeds because $\mathrm{SPAN}[6]=13=H$, so a circuit is detected.

\paragraph{L7 (contract to merge spans $12$ and $13$; third label)}
Increment the contraction label:
\[
SS\gets SS+1=14.
\]
Assign $\mathrm{SPAN}[J]\gets 14$ for all $J\le 7$ with $\mathrm{SPAN}[J]=H(=13)$ or $\mathrm{SPAN}[J]<0$.
This maps $\{6,7\}$ (span $13$) and $\{1,2,3,4,5\}$ (negative span $-12$) to a single component:
\[
\mathrm{SPAN}[1]=\cdots=\mathrm{SPAN}[7]=14.
\]
Control returns to L4 with active label $H=14$.

\paragraph{L4 (candidate search within active span; tableau pass)}
Now $H=14$, so the scanned columns are $J\in\{1,\dots,7\}$.
Admissible rows satisfy $\mathrm{SPAN}[I]\neq 14$, hence $I\in\{8,9,10\}$.
With current duals
\[
U_1[1..7]=(22,20,19,24,5,9,6),
\]
compute $\widetilde C(I,J)=C[I,J]-U_1[J]$:

{\footnotesize
\[
\begin{array}{c|ccccccc}
 & J=1 & J=2 & J=3 & J=4 & J=5 & J=6 & J=7\\\hline
I=8  & 58-22=36 & 30-20=10 & 33-19=14 & 43-24=19 & 95-5=90 & 28-9=19 & 25-6=19\\
I=9  & 73-22=51 & 55-20=35 & 64-19=45 & 43-24=19 & 69-5=64 & 42-9=33 & 81-6=75\\
I=10 & 89-22=67 & 61-20=41 & 97-19=78 & 63-24=39 & 25-5=20 & 26-9=17 & 71-6=65
\end{array}
\]
}
The minimum reduced cost is $10$ at $(I_1,J_1)=(8,2)$:
\[
DU=10,\qquad (I_1,J_1)=(8,2).
\]

\paragraph{L5 (dual update on the active span; tableau pass)}
Add $DU$ to all $J\le 7$ with $\mathrm{SPAN}[J]=14$ (i.e.\ $J=1,\dots,7$):
\[
U_1[1..7]\gets (32,30,29,34,15,19,16).
\]

\paragraph{L6--L8 (backward trace; tableau pass)}
Initialize the trace at
\[
J:=I_1=8.
\]
Since $I^{*}(8)=0$, the trace terminates immediately and no further contraction occurs.

\paragraph{L9--L10 (normalize and star with the full transfer chain)}
Since there are no negative span labels among $J\le 7$ after the contraction to span $14$, L9 makes no change.
Initialize
\[
(I_2,J_2):=(0,0).
\]
L10 then installs the entering predecessor in column $J_1=2$ and performs the required transfer chain.
The four transfer updates are:

\begin{enumerate}
\item Set $I^{*}(2)\gets 8$ (displacing $I_2=4$), and update the carried pair to $(I_1,J_1)\gets(3,1)$.
\item Set $I^{*}(1)\gets 3$ (displacing $I_2=6$), and update the carried pair to $(I_1,J_1)\gets(1,6)$.
\item Set $I^{*}(6)\gets 1$ (displacing $I_2=7$), and update the carried pair to $(I_1,J_1)\gets(6,7)$.
\item Set $I^{*}(7)\gets 6$ (displacing $I_2=0$), and the chain terminates.
\end{enumerate}

Thus, after the tableau installation and transfers, the starred predecessors on $j\le 7$ are
\[
(I^{*}(1),\dots,I^{*}(7))=(3,8,2,1,4,1,6).
\]

\paragraph{State after $K=7$}
At the end of $K=7$ (this $I^{*}(\cdot)$ configuration matches Figure~3 in the original paper),
\[
SS=14,\qquad \mathrm{SPAN}[1]=\cdots=\mathrm{SPAN}[7]=14,
\]
\[
U_1[1..7]=(32,30,29,34,15,19,16),
\]
\[
(I^{*}(1),\dots,I^{*}(7))=(3,8,2,1,4,1,6).
\]
Control returns to L3 to advance to $K=8$.

\subsection{\texorpdfstring{\(K=8\)}{K = 8}}

\paragraph{L3 (advance)}
Increment $K:=8$.  Since $K\neq J_0(=10)$, column $8$ is processed.  Set the active span label
\[
H:=\mathrm{SPAN}[8]=8.
\]
At entry to this stage (after $K=7$), we have
\[
U_1[1..7]=(32,30,29,34,15,19,16),\quad U_1[8]=0,
\]
\[
(I^{*}(1),\dots,I^{*}(7))=(3,8,2,1,4,1,6),\quad I^{*}(8)=0,
\]
\[
\mathrm{SPAN}[1]=\cdots=\mathrm{SPAN}[7]=14,\quad \mathrm{SPAN}[8]=8.
\]

\paragraph{L4 (candidate search within active span)}
L4 scans all columns $J\le K$ with $\mathrm{SPAN}[J]=H$.  Here this is only $J=8$.
Admissible rows satisfy $\mathrm{SPAN}[I]\neq H$ (so $I\neq 8$).  Because $U_1[8]=0$,
reduced costs coincide with raw costs:
\[
\widetilde C(I,8)=C[I,8]-U_1[8]=C[I,8].
\]
The minimum over $I\in\{1,\dots,10\}\setminus\{8\}$ is $6$ at $(I_1,J_1)=(9,8)$, hence
\[
DU=6,\qquad (I_1,J_1)=(9,8).
\]

\paragraph{L5 (dual update on the active span)}
Update all $J\le K$ with $\mathrm{SPAN}[J]=H$ (only $J=8$):
\[
U_1[8]\gets U_1[8]+DU=6.
\]

\paragraph{L6--L8 (backward trace and circuit test)}
Start the trace at the candidate tail:
\[
J:=I_1=9.
\]
Since $\mathrm{SPAN}[9]\neq H$ and $I^{*}(9)=0$, the trace terminates immediately and no contraction occurs.

\paragraph{L9--L10 (normalize and star)}
No span labels are negative for $J\le K$, so L9 makes no change and sets $I_2:=J_2:=0$.
L10 stars the entering candidate in column $8$:
\[
I^{*}(8)\gets 9,
\]
and the transfer chain ends because the previous star was empty.

\paragraph{State after $K=8$}
\[
U_1[1..8]=(32,30,29,34,15,19,16,6),
\qquad
(I^{*}(1),\dots,I^{*}(8))=(3,8,2,1,4,1,6,9),
\]
\[
\mathrm{SPAN}[1]=\cdots=\mathrm{SPAN}[7]=14,\quad \mathrm{SPAN}[8]=8.
\]

\subsection{\texorpdfstring{\(K=9\)}{K = 9} (two contractions and three L4 passes)}

\paragraph{L3 (advance)}
Increment \(K:=9\). Since \(K\neq J_0(=10)\), column \(9\) is processed. Set the active span label
\[
H:=\mathrm{SPAN}[9]=9.
\]
At entry to this stage, the relevant state is
\[
U_1[1..9]=(32,30,29,34,15,19,16,6,0),
\]
\[
(I^{*}(1),\dots,I^{*}(9))=(3,8,2,1,4,1,6,9,0),
\]
\[
\mathrm{SPAN}[1]=\cdots=\mathrm{SPAN}[7]=14,\qquad
\mathrm{SPAN}[8]=8,\qquad
\mathrm{SPAN}[9]=9,
\qquad SS=14.
\]

\paragraph{L4 (first candidate search)}
Only column \(9\) belongs to the active span. Since \(U_1[9]=0\), the reduced costs are
\[
\widetilde C(I,9)=C[I,9].
\]
The minimum over \(I\neq9\) is \(3\), attained at
\[
(I_1,J_1)=(8,9),\qquad DU=3.
\]

\paragraph{L5 (first column adjustment)}
Only column \(9\) receives the adjustment:
\[
U_1[9]\gets U_1[9]+DU=3.
\]

\paragraph{L6--L8 (first backward trace)}
The trace begins at
\[
J:=I_1=8.
\]
Since \(I^{*}(8)=9>0\), L8 is entered. The positive span label of node \(8\) is negated:
\[
\mathrm{SPAN}[8]\gets-8.
\]
Because the bar entry at node \(8\) is empty, the current candidate is recorded:
\[
(\bar I(8),\bar J(8))\gets(8,9).
\]
The trace then follows the starred predecessor:
\[
J\gets I^{*}(8)=9.
\]
At node \(9\), L7 succeeds because \(\mathrm{SPAN}[9]=H\). The candidate \(8\to9\) and the starred link \(9\to8\) therefore form a circuit.

\paragraph{L7 (first contraction)}
Increment the contraction label:
\[
SS\gets SS+1=15.
\]
Assign label \(15\) to the active span and the negatively marked span:
\[
\mathrm{SPAN}[8]=\mathrm{SPAN}[9]=15.
\]
Control returns to L4 with
\[
H:=\mathrm{SPAN}[9]=15.
\]

\paragraph{L4 (second candidate search)}
The active span now contains columns \(8\) and \(9\). The admissible initial nodes are
\[
I\in\{1,2,3,4,5,6,7,10\}.
\]
With \(U_1[8]=6\) and \(U_1[9]=3\), the reduced costs are

{\footnotesize
\[
\begin{array}{c|cc}
 & J=8 & J=9\\\hline
I=1  & 29-6=23 & 69-3=66\\
I=2  & 64-6=58 & 31-3=28\\
I=3  & 27-6=21 & 83-3=80\\
I=4  & 74-6=68 & 16-3=13\\
I=5  & 75-6=69 & 65-3=62\\
I=6  & 52-6=46 & 42-3=39\\
I=7  & 51-6=45 & 52-3=49\\
I=10 & 72-6=66 & 43-3=40
\end{array}
\]
}
The minimum reduced cost is \(13\), attained at
\[
(I_1,J_1)=(4,9),\qquad DU=13.
\]

\paragraph{L5 (second column adjustment)}
Both columns in the active span receive the adjustment:
\[
U_1[8]\gets6+13=19,\qquad
U_1[9]\gets3+13=16.
\]

\paragraph{L6--L8 (second backward trace)}
The trace begins at
\[
J:=I_1=4.
\]
Node \(4\) belongs to span \(14\) and has starred predecessor \(I^{*}(4)=1\). L8 therefore negates the entire span containing nodes \(1,\dots,7\):
\[
\mathrm{SPAN}[1]=\cdots=\mathrm{SPAN}[7]\gets-14.
\]
The bar entry at node \(4\) is not empty: it already contains
\[
(\bar I(4),\bar J(4))=(4,2).
\]
L8 consequently preserves this entry rather than replacing it by \((4,9)\). The trace follows the starred predecessor:
\[
J\gets I^{*}(4)=1.
\]

At node \(1\), the span is already negative and the bar entry is nonempty, so neither the span labels nor the bar arrays are changed. The trace continues:
\[
J\gets I^{*}(1)=3.
\]
The same holds at node \(3\), and the trace continues to
\[
J\gets I^{*}(3)=2.
\]
The bar entry at node \(2\) is empty, so L8 records the current candidate:
\[
(\bar I(2),\bar J(2))\gets(4,9).
\]
The trace then follows
\[
J\gets I^{*}(2)=8.
\]
Since
\[
\mathrm{SPAN}[8]=15=H,
\]
L7 succeeds. Thus, the trace returns to the active span and detects a second circuit.

\paragraph{L7 (second contraction)}
Increment the contraction label:
\[
SS\gets SS+1=16.
\]
The active span \(\{8,9\}\) and the negatively marked span \(\{1,\dots,7\}\) receive the common label \(16\):
\[
\mathrm{SPAN}[1]=\cdots=\mathrm{SPAN}[9]=16.
\]
The bar entries are not altered by the contraction. In particular,
\[
(\bar I(2),\bar J(2))=(4,9).
\]
Control returns to L4 with \(H=16\).

\paragraph{L4 (third candidate search)}
All processed columns \(J=1,\dots,9\) now belong to the active span. Node \(10\) is the only node outside it, so the search compares the links from node \(10\) into these columns. Before the next adjustment,
\[
U_1[1..9]=(32,30,29,34,15,19,16,19,16).
\]
The reduced costs are

{\footnotesize
\[
\begin{array}{c|ccccccccc}
J & 1&2&3&4&5&6&7&8&9\\\hline
C[10,J]-U_1[J]
&57&31&68&29&10&7&55&53&27
\end{array}
\]
}
The minimum is attained by \(10\to6\):
\[
(I_1,J_1)=(10,6),\qquad DU=7.
\]

\paragraph{L5 (third column adjustment)}
All columns \(1,\dots,9\) belong to the active span and receive the adjustment:
\[
U_1[1..9]\gets(39,37,36,41,22,26,23,26,23).
\]

\paragraph{L6--L8 (final backward trace)}
The trace begins at
\[
J:=I_1=10.
\]
Because \(I^{*}(10)=0\), the trace reaches an unstarred node immediately. No span is marked, and control passes to L9.

\paragraph{L9--L10 (star-and-bar transfer)}
L9 initializes
\[
(I_2,J_2):=(0,0).
\]
L10 then installs \(10\to6\) and propagates the displaced stars through four transfer iterations.

\smallskip
\noindent\textbf{Iteration 1.}
The current candidate is \((I_1,J_1)=(10,6)\). No bar entry refers to this pair. The procedure reads
\[
(I,J):=(\bar I(6),\bar J(6))=(6,1)
\]
and clears the bar entry at node \(6\):
\[
(\bar I(6),\bar J(6))\gets(0,0).
\]
It displaces the star \(1\to6\) and installs \(10\to6\):
\[
I_2\gets I^{*}(6)=1,\qquad I^{*}(6)\gets10.
\]
The displaced pair is \((I_2,J_2)=(1,6)\), and the next candidate is \((I_1,J_1)=(6,1)\).

\smallskip
\noindent\textbf{Iteration 2.}
No bar entry currently equals \((6,1)\). The procedure reads
\[
(I,J):=(\bar I(1),\bar J(1))=(4,2)
\]
and replaces this entry by the displaced pair:
\[
(\bar I(1),\bar J(1))\gets(1,6).
\]
It then displaces \(3\to1\) and installs \(6\to1\):
\[
I_2\gets I^{*}(1)=3,\qquad I^{*}(1)\gets6.
\]
The displaced pair becomes \((3,1)\), and the next candidate becomes \((4,2)\).

\smallskip
\noindent\textbf{Iteration 3.}
The bar entries at nodes \(3\) and \(4\), both of which refer to \((4,2)\), are replaced by the displaced pair:
\[
(\bar I(3),\bar J(3))\gets(3,1),\qquad
(\bar I(4),\bar J(4))\gets(3,1).
\]
The procedure then reads the bar entry at node \(2\):
\[
(I,J):=(\bar I(2),\bar J(2))=(4,9)
\]
and replaces it by
\[
(\bar I(2),\bar J(2))\gets(3,1).
\]
It displaces \(8\to2\) and installs \(4\to2\):
\[
I_2\gets I^{*}(2)=8,\qquad I^{*}(2)\gets4.
\]
The displaced pair becomes \((8,2)\), and the next candidate becomes \((4,9)\).

\smallskip
\noindent\textbf{Iteration 4.}
No remaining bar entry equals \((4,9)\). The bar entry at node \(9\) is empty, so the displaced pair is stored there:
\[
(\bar I(9),\bar J(9))\gets(8,2).
\]
The procedure then installs \(4\to9\):
\[
I_2\gets I^{*}(9)=0,\qquad I^{*}(9)\gets4.
\]
Since \(I_2=0\), the transfer chain terminates.

\paragraph{State after \(K=9\)}
The final state before the origin column is skipped is
\[
SS=16,\qquad
\mathrm{SPAN}[1]=\cdots=\mathrm{SPAN}[9]=16,
\]
\[
U_1[1..9]=(39,37,36,41,22,26,23,26,23),
\]
\[
(I^{*}(1),\dots,I^{*}(9))=(6,4,2,1,4,10,6,9,4).
\]
The nonempty bar entries are
\[
(\bar I(1),\bar J(1))=(1,6),
\]
\[
(\bar I(2),\bar J(2))
=(\bar I(3),\bar J(3))
=(\bar I(4),\bar J(4))
=(3,1),
\]
\[
(\bar I(5),\bar J(5))=(5,1),\qquad
(\bar I(7),\bar J(7))=(7,6),
\]
\[
(\bar I(8),\bar J(8))=(8,9),\qquad
(\bar I(9),\bar J(9))=(8,2).
\]
Control returns to L3.

\subsection{\texorpdfstring{\(K=10\)}{K = 10} (the origin is skipped and the algorithm terminates)}

\paragraph{L3 (advance and skip the origin)}
Control returns to L3 with $K=9$.  Increment $K:=10$.  Since $K=J_0(=10)$, the algorithm skips this column (no execution of L4--L10) and immediately returns to L3.

\paragraph{L3 (termination test)}
Now $K=N=10$, so the termination condition holds and control jumps to L99.

\paragraph{L99 (objective value)}
Since infeasibility was not detected, $Z$ is accumulated at L99 from the final starred predecessors:
\[
Z\gets \sum_{\substack{j=1\\ j\neq J_0}}^{N} C[I^{*}(j),j]
= \sum_{j=1}^{9} C[I^{*}(j),j].
\]
At termination (as in Bock's output), the starred predecessors are
\[
I^{*}(1)=6,\ I^{*}(2)=4,\ I^{*}(3)=2,\ I^{*}(4)=1,\ I^{*}(5)=4,\ I^{*}(6)=10,\ I^{*}(7)=6,\ I^{*}(8)=9,\ I^{*}(9)=4,
\]
so the selected arcs are
\[
6\to 1,\ 4\to 2,\ 2\to 3,\ 1\to 4,\ 4\to 5,\ 10\to 6,\ 6\to 7,\ 9\to 8,\ 4\to 9,
\]
and the objective value is
\[
\begin{aligned}
Z
&=C[6,1]+C[4,2]+C[2,3]+C[1,4]+C[4,5]+C[10,6]+C[6,7]+C[9,8]+C[4,9]\\
&=22+3+2+7+0+26+5+6+16\\
&=87.
\end{aligned}
\]

\paragraph{L100 (output)}
The algorithm reports \texttt{OPTIMUM SOLUTION} and \texttt{Z = 87}, and prints the final arrays $U_1$, $I^{*}$, $(\bar{I},\bar{J})$, and $\mathrm{SPAN}$.

\section{A structured reformulation}
\label{sec:reformulation}

Bock's Algol procedure uses statement labels, temporary sign changes of \(\mathrm{SPAN}\), and fresh numerical labels for contracted components. The accompanying implementation \texttt{bock\_algorithm\_readable} organizes the same computation into candidate-selection, column-adjustment, backward-trace, contraction, and transfer phases. It replaces the sign of \(\mathrm{SPAN}\) by an explicit trace mask and represents a contraction by assigning the active component's label to every component encountered by the trace. Since component labels serve only to determine whether two nodes belong to the same component, retaining the active label instead of creating a fresh label does not change the represented partition.

The implementation uses zero-based indices and \(-1\) for an empty entry. The presentation below retains the paper's one-based indices and writes \(\bot\) for an empty parent or candidate pair.

\subsection{Representation of the maintained state}
\label{sec:reformulated-state}

The correspondence between Bock's arrays and the readable implementation is summarized in Table~\ref{tab:reformulated-state}.

\begin{table}[t]
\centering
\small
\begin{tabular}{@{}lll@{}}
\toprule
Bock's representation & Readable representation & Role\\
\midrule
\(U_1[J]\)
  & \texttt{U1[J]}
  & Cumulative adjustment of column \(J\)\\
\(I^{*}[J]\)
  & \texttt{parents[J]}
  & Current provisional link entering \(J\)\\
\((\bar I[J],\bar J[J])\)
  & \texttt{candidate\_edges[J]}
  & Candidate reference used during transfer\\
\(\mathrm{SPAN}[J]\)
  & \texttt{components[J]}
  & Component containing \(J\)\\
Negative \(\mathrm{SPAN}\) labels
  & \texttt{cycle\_trace}
  & Components visited by the current trace\\
\(SS\)
  & Not required
  & Fresh labels are replaced by the active label\\
\(M\)
  & \texttt{np.inf}
  & Infeasible-link sentinel\\
\bottomrule
\end{tabular}
\caption{State correspondence between Bock's Algol procedure and the readable implementation.}
\label{tab:reformulated-state}
\end{table}

Two component arrays represent the same partition when
\[
\mathrm{SPAN}[I]=\mathrm{SPAN}[J]
\quad\Longleftrightarrow\quad
\mathrm{components}[I]=\mathrm{components}[J]
\]
for every pair of nodes \(I,J\). The numerical component labels themselves need not agree. Thus, assigning a new label \(SS\) to a contracted component and retaining the active component's existing label are equivalent whenever they induce the same partition.

The starred and barred arrays retain their original operational meanings. The assignment
\[
\mathrm{parents}[J]=I
\]
represents the provisional link \(I\to J\), corresponding to \(I^{*}[J]=I\). An entry
\[
\mathrm{candidate\_edges}[J]=(I_1,J_1)
\]
records the candidate current when the backward trace first traversed node \(J\) while its entry was empty. If the star entering \(J\) is later displaced, the recorded pair supplies the next candidate in the transfer chain.

\subsection{Active components and column adjustments}
\label{sec:reformulated-adjustments}

For the current column \(K\), let
\[
H=\mathrm{components}[K]
\]
be the active component. The candidate set is
\[
\mathcal{E}_K(H)
=
\left\{
(I,J):
1\leq J\leq K,\ 
\mathrm{components}[J]=H,\ 
\mathrm{components}[I]\neq H,\ 
C[I,J]<\infty
\right\}.
\]
The procedure selects the first pair \((I_1,J_1)\), in increasing order of \(J\) and then \(I\), attaining
\[
DU
=
\min_{(I,J)\in\mathcal{E}_K(H)}
\bigl(C[I,J]-U_1[J]\bigr).
\]
The column-major order and strict comparison reproduce the tie-breaking behavior of Bock's L4. If \(\mathcal{E}_K(H)\) is empty, the instance is reported infeasible.

The adjustment phase sets
\[
U_1[J]\gets U_1[J]+DU
\]
for every \(J\leq K\) belonging to the active component. If the active components updated during the execution are \(H_1,H_2,\dots\), with corresponding increments \(\Delta_1,\Delta_2,\dots\), then
\[
U_1[J]
=
\sum_{t:\,J\in H_t}\Delta_t.
\]
Thus, \(U_1[J]\) aggregates all adjustments applied to components containing \(J\).

This aggregate has a local reduced-cost interpretation. In a set-based dual view of the arborescence problem, an adjustment associated with a component contributes to the constraint for \(I\to J\) only when \(J\) belongs to the component and \(I\) does not. The corresponding contribution would therefore have the form
\[
\sum_{t:\,J\in H_t,\ I\notin H_t}\Delta_t.
\]
For a link considered while entering the current active component, earlier components containing \(J\) are nested within the current component, while \(I\) lies outside it. The two sums consequently agree during that candidate search. For an arbitrary link at a later stage, however, \(U_1[J]\) may also contain adjustments associated with components containing both \(I\) and \(J\). The quantity \(C[I,J]-U_1[J]\) should therefore be understood as Bock's operational reduced cost for the current search, not as the final slack of an independently represented dual constraint.

\begin{lemma}[Candidate tightness]
\label{lem:candidate-tightness}
Suppose that \(\mathcal{E}_K(H)\neq\emptyset\). Immediately after the column adjustment, the selected candidate satisfies
\[
C[I_1,J_1]-U_1[J_1]=0,
\]
and every link in \(\mathcal{E}_K(H)\) has nonnegative reduced cost in the resulting search state.
\end{lemma}

\begin{proof}
Before the adjustment, the definition of \(DU\) gives
\[
C[I,J]-U_1[J]\geq DU
\]
for every \((I,J)\in\mathcal{E}_K(H)\), with equality for the selected pair \((I_1,J_1)\). Every terminal node \(J\) occurring in \(\mathcal{E}_K(H)\) belongs to the active component and therefore receives the same increment \(DU\). Its reduced cost after the update is
\[
C[I,J]-\bigl(U_1[J]+DU\bigr)
=
\bigl(C[I,J]-U_1[J]\bigr)-DU
\geq0.
\]
Equality holds for \((I_1,J_1)\).
\end{proof}

The lemma concerns the state immediately following the current update. A later adjustment to a larger contracted component can change the numerical value \(C[I,J]-U_1[J]\) for a previously selected link.

\subsection{Phase-structured procedure}
\label{sec:structured-procedure}

Algorithm~\ref{alg:structured-bock} gives the phase organization implemented by \texttt{bock\_algorithm\_readable}. It uses Bock's candidate scan and transfer assignments but replaces negative component labels with the set \(\mathrm{TRACE}\) of processed nodes belonging to components visited during the backward trace.

\begin{algorithm}
\tiny
\caption{Structured reformulation of Bock's algorithm}
\label{alg:structured-bock}
\begin{algorithmic}[1]
\REQUIRE Number of nodes \(N\); origin \(j_0\); cost matrix \(C\), with \(C[I,J]=\infty\) for an infeasible link
\ENSURE Parent array defining a minimum directed spanning tree, or \textsc{Infeasible}

\STATE Set \(U_1[J]\gets0\), \(\mathrm{parents}[J]\gets\bot\), and
\(\mathrm{candidate\_edges}[J]\gets\bot\) for every \(J\)
\STATE Set \(\mathrm{components}[J]\gets J\) for every \(J\)

\FOR{\(K=1\) \TO \(N\)}
    \IF{\(K=j_0\)}
        \STATE Continue with the next column
    \ELSE
        \STATE Set \(\mathrm{cycle\_possible}\gets\mathrm{true}\)

        \WHILE{\(\mathrm{cycle\_possible}\)}
            \STATE Set \(H\gets\mathrm{components}[K]\)
            \STATE Form
            \[
            \mathcal{E}_K(H)
            =
            \{(I,J):
            J\leq K,\ 
            \mathrm{components}[J]=H,\ 
            \mathrm{components}[I]\neq H,\ 
            C[I,J]<\infty\}
            \]

            \IF{\(\mathcal{E}_K(H)=\emptyset\)}
                \RETURN \textsc{Infeasible}
            \ENDIF

            \STATE Select the first minimum in increasing order of \(J\) and then \(I\):
            \[
            (I_1,J_1)
            \in
            \arg\min_{(I,J)\in\mathcal{E}_K(H)}
            \bigl(C[I,J]-U_1[J]\bigr)
            \]
            \STATE Set
            \[
            DU\gets C[I_1,J_1]-U_1[J_1]
            \]

            \FOR{each \(J\leq K\) satisfying \(\mathrm{components}[J]=H\)}
                \STATE Set \(U_1[J]\gets U_1[J]+DU\)
            \ENDFOR

            \STATE Set \(J\gets I_1\) and \(\mathrm{TRACE}\gets\emptyset\)

            \WHILE{\(\mathrm{parents}[J]\neq\bot\) and
                    \(\mathrm{components}[J]\neq H\)}
                \STATE Add every \(J_2\leq K\) satisfying
                \(\mathrm{components}[J_2]=\mathrm{components}[J]\)
                to \(\mathrm{TRACE}\)

                \IF{\(\mathrm{candidate\_edges}[J]=\bot\)}
                    \STATE Set
                    \[
                    \mathrm{candidate\_edges}[J]\gets(I_1,J_1)
                    \]
                \ENDIF

                \STATE Set \(J\gets\mathrm{parents}[J]\)
            \ENDWHILE

            \IF{\(\mathrm{components}[J]=H\)}
                \FOR{each \(J_2\in\mathrm{TRACE}\)}
                    \STATE Set \(\mathrm{components}[J_2]\gets H\)
                \ENDFOR
                \STATE Continue the candidate-search loop
            \ELSE
                \STATE Set \(\mathrm{cycle\_possible}\gets\mathrm{false}\)
            \ENDIF
        \ENDWHILE

        \STATE Set \((I_2,J_2)\gets(\bot,\bot)\)

        \REPEAT
            \FOR{each \(J\leq K\) satisfying
            \(\mathrm{candidate\_edges}[J]=(I_1,J_1)\)}
                \STATE Set
                \[
                \mathrm{candidate\_edges}[J]\gets(I_2,J_2)
                \]
            \ENDFOR

            \STATE Set
            \[
            (I,J)\gets\mathrm{candidate\_edges}[J_1]
            \]
            \STATE Set
            \[
            \mathrm{candidate\_edges}[J_1]\gets(I_2,J_2)
            \]
            \STATE Set
            \[
            I_2\gets\mathrm{parents}[J_1],
            \qquad
            \mathrm{parents}[J_1]\gets I_1
            \]

            \IF{\(I_2\neq\bot\)}
                \STATE Set
                \[
                (J_2,I_1,J_1)\gets(J_1,I,J)
                \]
            \ENDIF
        \UNTIL{\(I_2=\bot\)}
    \ENDIF
\ENDFOR

\STATE Compute
\[
Z\gets\sum_{J\neq j_0}C[\mathrm{parents}[J],J]
\]
\RETURN \(Z\) and \(\mathrm{parents}\)
\end{algorithmic}
\end{algorithm}

\subsection{Circuit contraction}
\label{sec:structured-contraction}

The backward trace begins at the initial node \(I_1\) of the candidate. Whenever it traverses a node outside the active component, it records the entire component containing that node in \(\mathrm{TRACE}\). If the trace reaches a node belonging to the active component, the candidate and the traversed starred links close a directed circuit. All components recorded in \(\mathrm{TRACE}\) are then merged into the active component.

In the Algol procedure, the same operation is represented by negating the labels of the visited components, incrementing \(SS\), and assigning the new label to the active component and all negatively marked components. The readable procedure leaves the labels unchanged during the trace and, upon detecting the circuit, assigns the active label directly to the marked nodes. Both operations produce the same partition.

\begin{lemma}[Progress under contraction]
\label{lem:contraction-progress}
Every contraction merges the active component with at least one distinct component. It therefore strictly decreases the number of components represented among the processed nodes.
\end{lemma}

\begin{proof}
Candidate selection requires
\[
\mathrm{components}[I_1]\neq H,
\]
so the trace begins outside the active component. A contraction occurs only if the trace subsequently reaches a node whose component is \(H\). Before doing so, the trace traverses at least one node belonging to a component distinct from \(H\), and that component is recorded in \(\mathrm{TRACE}\). The contraction assigns the common label \(H\) to the active component and every recorded component. Hence, at least two previously distinct components are replaced by one.
\end{proof}

\begin{corollary}
\label{cor:contraction-bound}
The procedure performs at most \(N-1\) contractions.
\end{corollary}

\begin{proof}
The initial state contains \(N\) singleton components. By Lemma~\ref{lem:contraction-progress}, every contraction strictly decreases the number of represented components, and no operation divides a component after it has been formed.
\end{proof}

This bound concerns the number of contractions, not the total running time. Candidate search may be repeated after each contraction, and its cumulative cost is not determined by the contraction count alone.

\subsection{Star-and-bar transfer}
\label{sec:structured-transfer}

If the trace reaches a node without a starred predecessor, the candidate can be inserted without forming another circuit at the current component level. The transfer begins with the empty displaced pair
\[
(I_2,J_2)=(\bot,\bot).
\]
For the current candidate \((I_1,J_1)\), every stored reference to that pair is redirected to the displaced pair. The procedure then reads the candidate stored at \(J_1\), replaces that entry by the displaced pair, and installs \(I_1\to J_1\) as the new star. If a previous star \(I_2\to J_1\) has been displaced, the stored pair becomes the next candidate and the operation repeats. The transfer ends when the installed candidate displaces no previous star.

The global replacement
\[
\mathrm{candidate\_edges}[J]=(I_1,J_1)
\quad\Longrightarrow\quad
\mathrm{candidate\_edges}[J]\gets(I_2,J_2)
\]
is essential. A bar is a reference to a candidate pair rather than an independent copy of an edge choice. When that candidate is transferred, all references to it must be redirected to the pair it displaced.

\subsection{Operational correspondence}
\label{sec:operational-correspondence}

To compare the Algol procedure and the structured reformulation, convert the implementation's zero-based indices to the paper's one-based indices and identify \(-1\) with the empty value \(0\) used by Bock. States at corresponding phase boundaries are called \emph{aligned} when:

\begin{enumerate}
\item their column adjustments agree;
\item their starred predecessors and candidate pairs agree after the index and empty-value conversion; and
\item their component arrays induce the same partition, regardless of the numerical labels used for the components.
\end{enumerate}

During a backward trace, alignment additionally identifies the nodes carrying negative \(\mathrm{SPAN}\) labels with the nodes marked by \(\mathrm{TRACE}\).

\begin{theorem}[Operational correspondence]
\label{thm:operational-correspondence}
Suppose that Bock's Algol procedure and Algorithm~\ref{alg:structured-bock} receive the same cost matrix and origin, with \(M\) in the Algol procedure corresponding to \(+\infty\) in the structured formulation. Suppose further that both use increasing \(J\), then increasing \(I\), and retain the first candidate attaining the minimum reduced cost. Starting from aligned initial states, the two procedures remain aligned at every corresponding phase boundary. They perform the same column adjustments, represent the same contractions, execute the same star-and-bar transfers, and terminate with the same feasibility decision, objective value, and starred predecessor array.
\end{theorem}

\begin{proof}
The initial states are aligned. Both procedures initialize the column adjustments to zero, leave every starred and barred entry empty, and represent every node as a singleton component.

Assume that the states are aligned at the beginning of a candidate-search phase. Since their component arrays induce the same partition, they examine the same columns in the active component and the same initial nodes outside it. Their column adjustments also agree. They therefore compute the same reduced costs over the same candidate set. The common scan order and first-minimum rule yield the same \((I_1,J_1)\) and \(DU\). Both procedures then add \(DU\) to precisely the processed columns belonging to the active component.

The backward traces begin at the same node \(I_1\) and follow the same starred predecessors. Whenever the trace visits a component outside the active component, the Algol procedure negates its \(\mathrm{SPAN}\) label, while the structured procedure marks the same component in \(\mathrm{TRACE}\). Repeatedly marking a component has no additional effect in either representation. Since the bar arrays are aligned, both procedures also record \((I_1,J_1)\) at the same trace nodes and preserve the same nonempty entries.

The traces consequently stop under the same condition. If they return to the active component, the Algol procedure assigns a fresh label to the active and negatively marked components, whereas the structured procedure assigns the active label to the active and trace-marked components. Although the numerical labels differ, the resulting component partitions are identical. The procedures therefore resume candidate selection from aligned states. If the traces instead reach an empty starred predecessor, the Algol procedure restores the negative labels and the structured procedure discards \(\mathrm{TRACE}\); neither operation changes the component partition.

The transfer phases begin with the same candidate and the same empty displaced pair. In each transfer iteration, both procedures replace every bar referring to the current candidate by the displaced pair, read and replace the bar at \(J_1\), displace the same starred predecessor, and install the same candidate. If the displaced predecessor is nonempty, both assign the same next candidate and repeat. If it is empty, both terminate the transfer and advance to the next column. Induction over the transfer iterations therefore preserves alignment.

Induction over candidate searches, contractions, transfers, and column advances establishes alignment at every corresponding phase boundary. If the candidate set is empty, both procedures report infeasibility. Otherwise, they terminate with the same starred predecessor for every nonorigin node and hence compute the same objective value.
\end{proof}

\begin{corollary}[Correctness of the reformulation]
\label{cor:reformulation-correctness}
Under Bock's input assumptions, Algorithm~\ref{alg:structured-bock} returns the same minimum-cost arborescence as the original procedure, including its scan-order resolution of ties, or reports the same instance infeasible.
\end{corollary}

\begin{proof}
Theorem~\ref{thm:operational-correspondence} establishes that the structured procedure and Bock's procedure produce the same starred predecessor array and feasibility decision. The result therefore follows from the correctness of Bock's original algorithm.
\end{proof}

\section{Conclusion}
\label{sec:conclusion}

This paper has provided a self-contained tutorial on Bock's 1971 algorithm for minimum directed spanning trees. The presentation preserves the matrix orientation, variables, statement labels, tests, assignment order, and control flow of the original Algol procedure while explaining the roles of its column adjustments, starred predecessors, barred references, and span labels. It thereby makes explicit how reduced-cost selection, backward tracing, circuit contraction, and star-and-bar transfer jointly construct a minimum-cost arborescence.

The three-node example isolates the central mechanism of the algorithm: least-cost provisional links form a circuit, the circuit is represented by a common span label, and a link entering the contracted span initiates the transfers that resolve the provisional selections. The complete trace of Bock's original ten-node matrix follows every candidate search, column adjustment, backward trace, contraction, and transfer from initialization to the reported optimum. Together, these examples connect the individual Algol statements to the evolving state of the computation.

The structured reformulation expresses the same computation through explicit phases and component-based state. Its readable implementation replaces temporary sign changes of \(\mathrm{SPAN}\) with an explicit trace mask and replaces fresh contraction labels with the persistent label of the active component. The candidate-tightness result characterizes the immediate effect of each column adjustment, while the contraction result establishes strict progress in the component partition. The operational-correspondence theorem shows that, under the same scan order and tie-breaking convention, the reformulated procedure and the Algol listing select the same candidates, apply the same adjustments and transfers, represent the same contracted components, and return the same arborescence and objective value.

The line-by-line account and the structured reformulation provide complementary descriptions of Bock's method. The former makes the original program directly traceable; the latter separates its recurring operations and establishes their correspondence independently of the numerical component labels used to represent contractions. The resulting account supports faithful implementation and verification while preserving the distinctive computational organization of Bock's original algorithm.


\end{document}